\documentclass[]{fairmeta}
\usepackage{graphicx}
\usepackage{amsmath}
\usepackage{amssymb}
\usepackage{amsthm}
\newtheorem{proposition}{Proposition}
\newtheorem{remark}{Remark}
\usepackage{enumitem}
\usepackage{graphicx}
\usepackage{wrapfig}
\usepackage{subcaption}
\usepackage{multirow}
\usepackage{pifont}
\usepackage[table]{xcolor}
\usepackage[dvipsnames]{xcolor}
\definecolor{DeepBlue}{RGB}{0,51,102}
\newcommand{\cbtcp}[1]{\tcp{\textcolor{DeepBlue}{#1}}}
\usepackage[ruled,vlined,linesnumbered]{algorithm2e}
\usepackage{titletoc}  
\usepackage{tocloft}
\newcommand\DoToC{%
  \startcontents
  \printcontents{}{1}{\textbf{Table of Contents (Appendix)}\vskip9pt\hrule\vskip5pt}
  \vskip3pt\hrule\vskip5pt
}

\title{DUEL: Adversarial Self-Play for Multimodal Reasoning}

\author[1]{Lin Qiu}
\author[1] {Hanqing Zeng}
\author[1] {Yao Liu}
\author[1] {Bingjun Sun}
\author [1] {Guangdeng Liao}
\author[1]{Ji Liu}

\affiliation[1]{Meta AI}

\abstract{Reinforcement learning (RL) has emerged as an effective paradigm for improving the reasoning capability of vision-language models (VLMs). However, RL-based optimization typically depends on costly high-quality annotations that are difficult to scale. Existing unsupervised alternatives may drift toward biased solutions due to weak visual grounding and the lack of reliable verification signals. We propose a self-evolving post-training framework, DUEL, where supervision emerges from adversarial interactions between two policies initialized from the same pretrained VLM. A Challenger generates an image-grounded true claim together with a minimally perturbed hard-negative counterpart, while a Solver verifies both claims against the image, encouraging fine-grained visual discrimination under near-neighbor semantics. To stabilize optimization, we introduce a length-normalized log-likelihood reward that preserves informative optimization signals beyond binary outcome supervision and improves learning stability under sparse feedback. Experiments show that DUEL consistently improves visual reasoning and robust discrimination without additional human annotations, external reward models, or image editing tools.}

\date{\today}
\correspondence{\email{lin.qiu.stats@gmail.com}}

\begin{document}

\maketitle

\begin{figure}[th] \centering \includegraphics[width=1\linewidth]{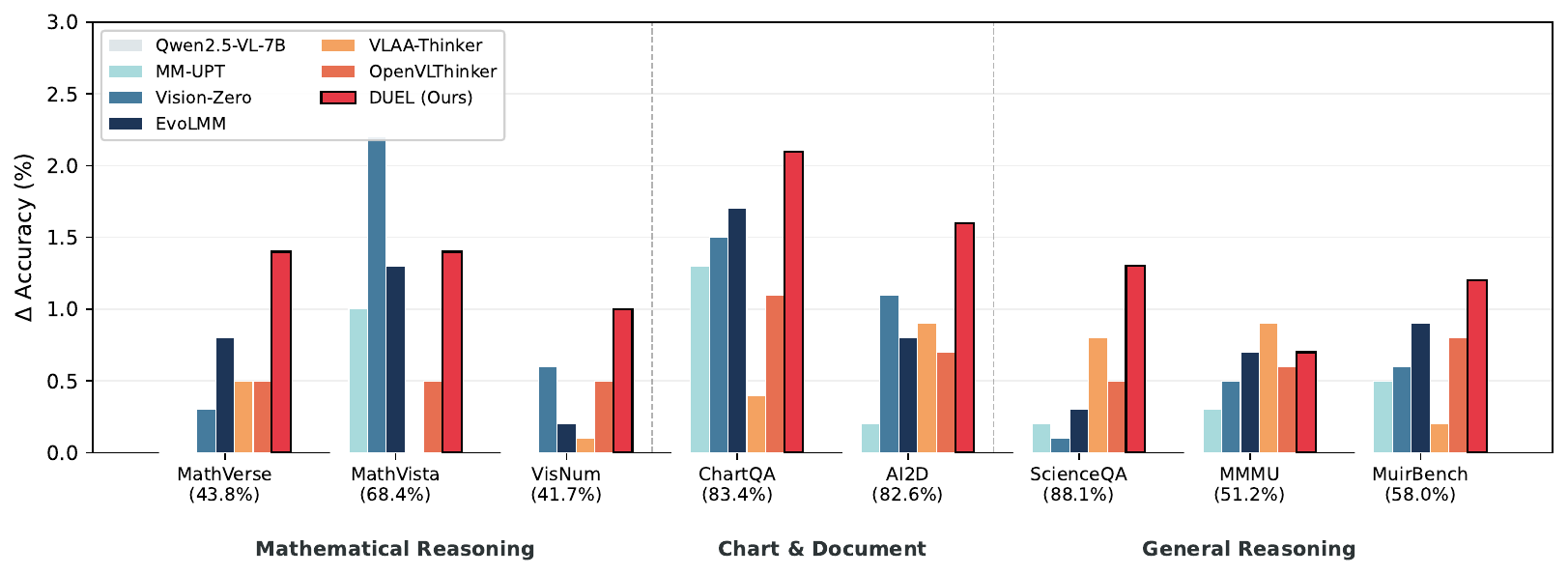} \caption{\textbf{Performance comparison of DUEL with SOTA post-training methods for VLMs.} All methods are trained on the same base model. The horizontal axis shows each benchmark with the base model's accuracy in parentheses; the vertical axis shows accuracy improvement ($\Delta$\%). Benchmarks are grouped into three categories: Mathematical Reasoning, Chart \& Document understanding, and General Reasoning. DUEL demonstrating broad and consistent improvement across all task categories without any human annotations.} \vspace{-2pt} \label{fig:comparsion_overall} \end{figure}

\section{Introduction} 
Vision-language models (VLMs) have achieved strong performance on multimodal tasks including image captioning~\citep{li2022blip}, visual question answering~\citep{alayrac2022flamingo}, and multimodal reasoning~\citep{chen2022pali}. Yet most training paradigms depend on large-scale human-curated data or external supervision such as supervised fine-tuning~\citep{dai2023instructblip} and preference-based alignment~\citep{yu2024rlhf}, constraining scalability and introducing reward bias in open-ended visual environments. Self-evolution has been widely adopted for LLMs, where models generate their own training signals through self-play~\citep{liu2025breaking}, self-critique~\citep{yuan2024self}, and iterative preference optimization~\citep{rafailov2023direct}. Absolute Zero~\citep{zhao2025absolute} exemplifies this by learning to propose and solve tasks without external data. Extending self-evolution to VLMs is increasingly urgent given the cost of multimodal annotation, yet existing approaches face fundamental limitations: self-consistency methods~\citep{visplay,EvoLMM} can reinforce confidently incorrect predictions and plateau, while Vision-Zero~\citep{vision-zero} relies on external image editors (GPT-based or Nano Banana modules) to construct training signals. Both lack a mechanism to ground rewards in visual evidence without external tools. 

\begin{figure*}[t]
    \centering
    \includegraphics[
        width=\textwidth,
        trim=0 160 0 140,
        clip
    ]{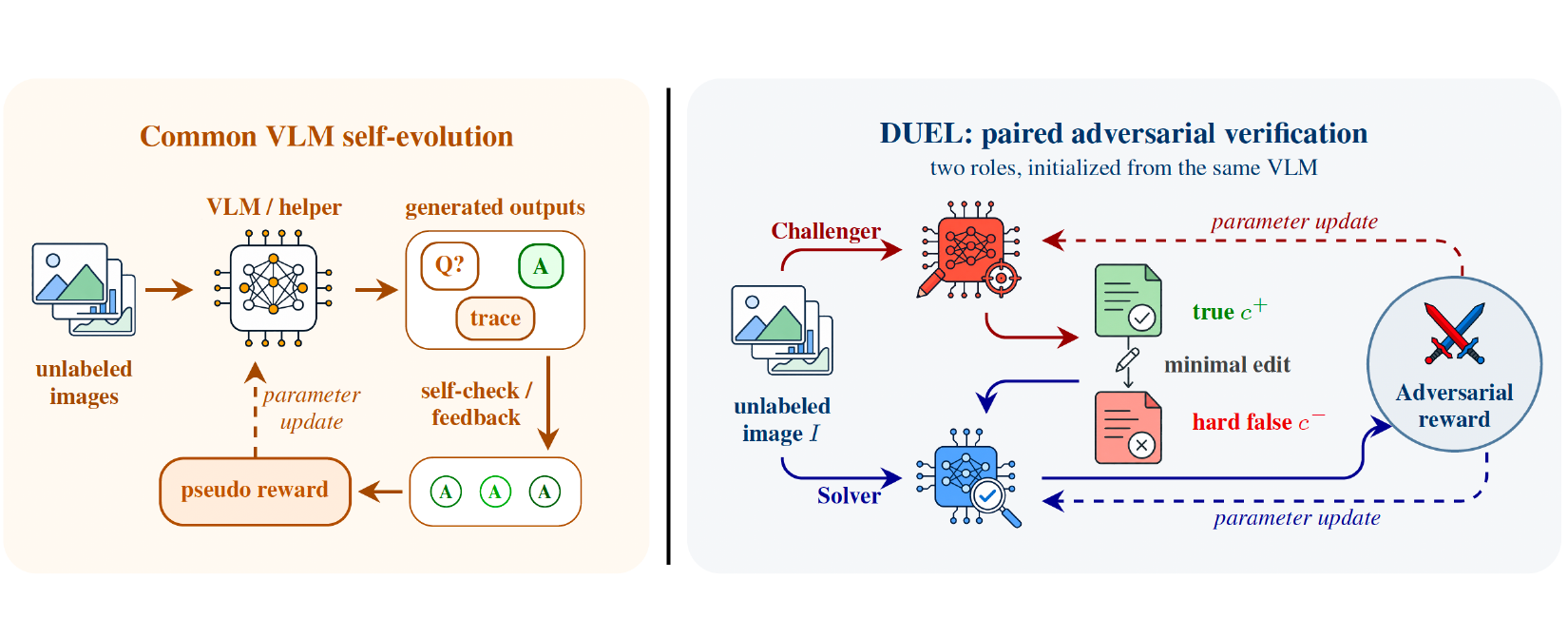}
    
    \caption{\textbf{DUEL compared with common VLM self-evolution flows.}
    Prior self-play and self-evolving VLM methods typically generate questions,
    answers, or rationales from unlabeled images and derive pseudo-rewards via
    agreement, self-checking, or tool feedback. In contrast, \textsc{DUEL}
    employs two adversarial roles: a Challenger generates a true claim and a
    minimally edited hard negative, while a Solver verifies both against the
    image. The resulting adversarial outcome updates both agents, providing
    grounded supervision without labels, teachers, external verifiers, or image
    editing. Fig.~\ref{fig:srta_framework} illustrates the overall workflow.}

    \label{fig:pair_supervision_source}
\end{figure*}


The central challenge is constructing scalable training signals that remain grounded in visual evidence without relying on additional human annotations or external reward supervision. To this end, we propose \textbf{DUEL}, which derives supervision entirely from adversarial interactions between two policies instantiated from the same pretrained VLM. DUEL follows a two-stage workflow: 
\begin{enumerate}[leftmargin=*,itemsep=2pt,topsep=2pt] \item \textbf{Adversarial Paired Claim Generation}: A Challenger produces an image-grounded true claim and a minimally perturbed hard-negative, constructing near-neighbor supervision that cannot be resolved through language priors alone. \item \textbf{Calibrated Claim Verification}: A Solver verifies claim truthfulness under a length-normalized likelihood reward that promotes consistently confident correctness while penalizing confident errors. \end{enumerate} 
By coupling near-neighbor adversarial supervision with calibrated rewards, DUEL turns unlabeled images into reliable training signals that tightly ground learning in visual evidence, requiring no external annotations, teacher models, verifiers, or image transformations. The main contributions of this paper are: 

\noindent\textbullet\ \textbf{New Perspective.} We formulate self-evolving VLM reasoning as an adversarial verification game on unlabeled images, deriving training signals from adversarial outcome verification rather than additional human supervision or self-agreement. 

\noindent\textbullet\ \textbf{Adversarial Self-Play Framework.} We propose DUEL, a Challenger--Solver paradigm with near-neighbor paired claims and a confidence-calibrated reward, enabling fine-grained visual discrimination through zero-sum outcome-based optimization. 

\noindent\textbullet\ \textbf{Theoretical Grounding.} We prove that (i) the adversarial game admits a Nash equilibrium under standard assumptions; (ii) near-neighbor negatives theoretically encourage higher mutual dependence between Solver decisions and visual evidence; and (iii) the adversarial objective induces an adaptive curriculum where task difficulty increases with Solver competence. 

\noindent\textbullet\ \textbf{Empirical Validation.} We conduct extensive experiments on fine-grained visual reasoning and robust discrimination benchmarks, demonstrating consistent gains and improved stability.

\section{Related Work}

\textbf{Supervised Multimodal Pretraining.}
Early vision-language models were primarily trained under supervised learning paradigms, relying on large-scale human-annotated datasets \citep{LXMERT, UNITER}. Subsequent joint vision-language pretraining approaches \citep{VISUALBERT, ViLBERT, unicoder-VL} aligned visual and textual representations through cross-modal encoders. CLIP \citep{CLIP} further advanced this direction via large-scale contrastive learning on web-scale image-text pairs, significantly improving transferability. Building on these advances, Flamingo \citep{alayrac2022flamingo} and BLIP-2 \citep{Q-Former} extended large language models to multimodal settings using cross-attention and lightweight bridging modules. Despite their success, these approaches remain heavily dependent on curated data or high-quality supervision.

\noindent\textbf{RLHF-Based Multimodal Alignment.}
Reinforcement Learning from Human Feedback (RLHF) has become a dominant paradigm for aligning large language models \citep{PPO, chatgpt}, and has been extended to multimodal settings. Methods such as Factually Augmented RLHF \citep{Align} train reward models using human preference data to improve factual grounding, while DPO \citep{rafailov2023direct} and related approaches directly optimize policies from preference comparisons without explicit reward modeling. However, these methods rely on externally provided preference pairs or static preference signals. In contrast, DUEL constructs training signals online through adversarial self-play on unlabeled images.

\noindent\textbf{Self-Play and Self-Evolving Learning Paradigms.} Recent work reduces human supervision by leveraging automated training signals. LLaVA \citep{LLaVA} synthesizes multimodal instruction data using strong language models, while VLM-RM \citep{VLM-RM}, RL-VLM-F \citep{RL-VLM-F}, and Eureka \citep{eureka} construct or optimize reward functions with foundation models. More closely related to our work, Vision-Zero \citep{vision-zero}, EvoLMM \citep{EvoLMM}, and VisPlay \citep{visplay} explore self-play and self-consistency mechanisms for learning from unlabeled images. However, these methods primarily rely on agreement-based or consistency-based signals, which may reinforce biased predictions and provide weak visual grounding. In contrast, \textsc{DUEL} formulates self-play as an adversarial paired verification game, where a Challenger generates near-neighbor counterfactual claims and a Solver verifies them against the image, encouraging fine-grained visually grounded discrimination.
\section{Method}
\label{sec:method}

\textbf{Problem Formulation.}
Let $\mathcal{D}$ denote an unlabeled image distribution and let $I \sim \mathcal{D}$ be a sampled image.
DUEL initializes two policies from the same pretrained VLM: a \emph{Challenger} $\pi_{\phi}(c \mid I,z)$ that generates an image-grounded claim $c$ conditioned on a polarity variable $z\in\{1,0\}$, and a \emph{Solver} $\pi_{\theta}(s \mid I,c)$ that outputs a verification sequence $s$ and a decision $a=h(s)\in\{\texttt{yes},\texttt{no}\}$.
Here, $z=1$ indicates that the generated claim should be true and $z=0$ indicates that it should be false.
In each episode, the Challenger constructs a paired instance $(c^{+},c^{-})$ by sampling $c^{+}\sim\pi_{\phi}(\cdot\mid I,z{=}1)$ and then $c^{-}\sim\pi_{\phi}(\cdot\mid I,c^{+},z{=}0)$, and the Solver is queried on both claims with fixed targets $y^{+}=\texttt{yes}$ and $y^{-}=\texttt{no}$.
No human labels or external verifiers are used. The objective of DUEL is to achieve self-supervised improvement of image-grounded verification via paired adversarial self-play.


\begin{figure*}[t]
  \centering
  \includegraphics[width=\textwidth]{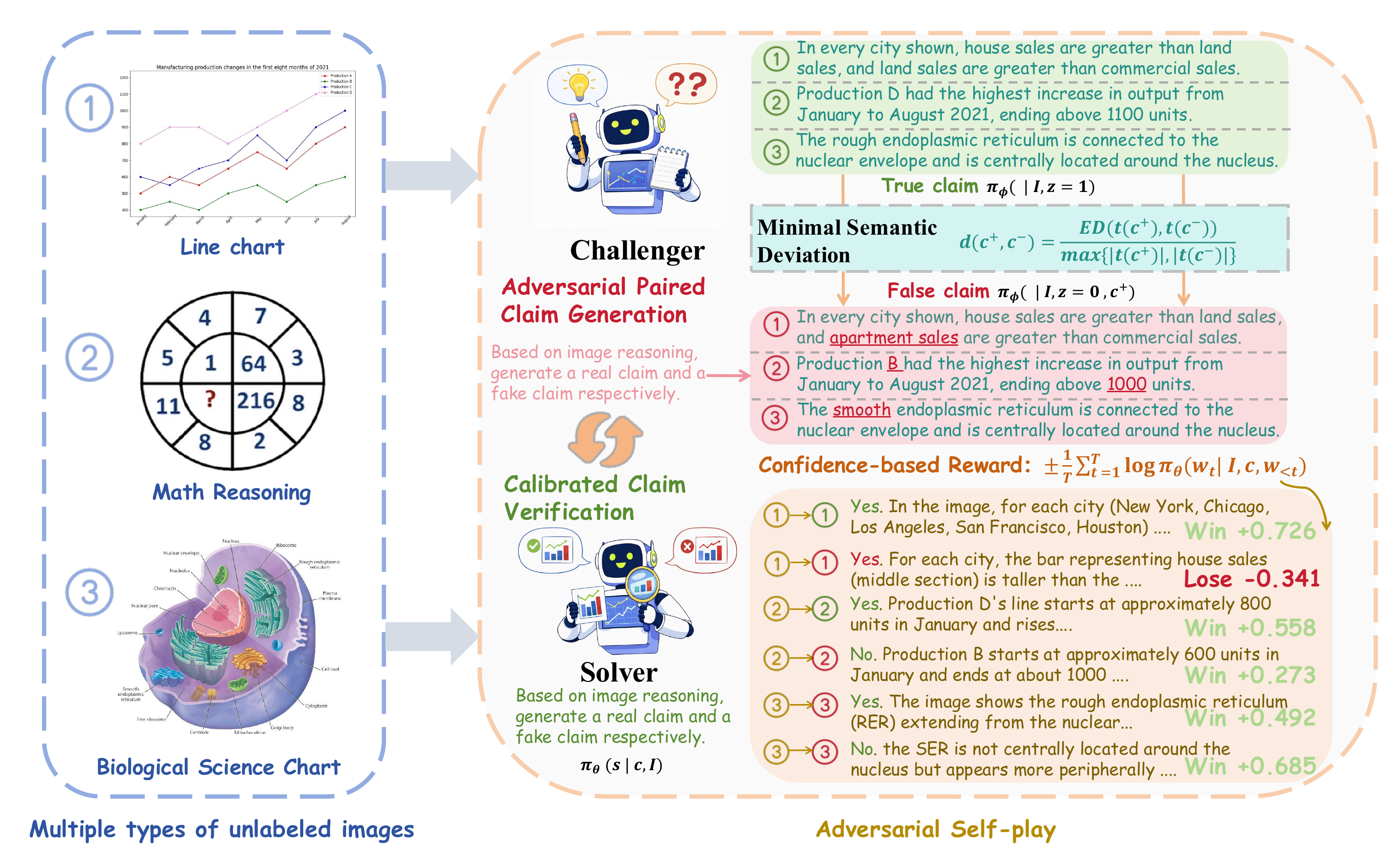}
  \caption{\textbf{Overall framework of DUEL.} Given an unlabeled image, the Challenger first generates an image-supported true claim and then constructs a minimally perturbed hard-negative false claim. The Solver verifies each claim against the image, and an outcome-based confidence reward provides the training signal to update both agents through adversarial self-play.}
  \label{fig:srta_framework}
\end{figure*}

In this section, we present DUEL, a self-evolving framework for training VLMs on unlabeled images via adversarial verification. DUEL instantiates a Challenger to generate an image-grounded true claim and a minimally perturbed hard-negative, and a Solver to verify claim truthfulness with a calibrated likelihood-based reward. An overview of DUEL is shown in Fig.~\ref{fig:srta_framework}.

\subsection{Adversarial Paired Claim Generation}
\label{sec:episode}
Unsupervised self-evolution signals often suffer from weak visual grounding and reward bias, allowing models to exploit language priors without resolving fine-grained visual evidence \citep{zhou2024calibrated}. DUEL instead generates a true claim paired with a minimally perturbed hard-negative, forcing near-neighbor discrimination and encouraging the model to rely more on visual evidence under tightly controlled semantics.

The Challenger is first prompted to generate a \emph{true} image-grounded claim that requires visual reasoning evidence from the image $I$. It then samples an output sequence $o^{+}$ from its policy:
\begin{equation}
o^{+} \sim \pi_{\phi}(\cdot \mid I, z=1),
\end{equation}
where $z \in \{0,1\}$ is a conditioning variable indicating whether the claim should be true ($z=1$) or false ($z=0$).
We deterministically extract the claim text $c^{+}$ from $o^{+}$ via a parsing function $g(\cdot)$:
\begin{equation}
c^{+} = g(o^{+}).
\end{equation}
For interpretability, $o^{+}$ also includes an image-evidence explanation $r^{+}$ that justifies why the claim holds for $I$.

\noindent\textbf{Paired hard-negative (false) claim generation.}
A semantically unconstrained negative may be rejected by language plausibility alone, weakening supervision and bypassing visual evidence \citep{goyal2017making}. We therefore enforce minimal semantic deviation to induce near-neighbor false claims.
To enforce \emph{minimal semantic deviation} and construct an \emph{adversarial paired hard negative}, the Challenger generates the false claim conditioned on both the image $I$ and the previously generated true claim $c^{+}$:
\begin{equation}
o^{-} \sim \pi_{\phi}(\cdot \mid I, c^{+}, z=0), \qquad
c^{-} = g(o^{-}).
\end{equation}
This conditioning restricts the false claim to be a subtle modification of the true claim, thereby preventing trivial falsehoods and promoting fine-grained visual reasoning.
We implement the minimal-deviation constraint with a token-level edit similarity.
Let $\mathbf{t}(c)$ denote the token sequence obtained from claim $c$ after normalization and tokenization.
Let $\mathrm{ED}(\mathbf{t}(c^{+}),\mathbf{t}(c^{-}))$ denote the minimum number of insertions, deletions, and substitutions that transforms $\mathbf{t}(c^{+})$ into $\mathbf{t}(c^{-})$.
We use a length-normalized edit distance
\begin{equation}
d(c^{+},c^{-}) =
\frac{\mathrm{ED}(\mathbf{t}(c^{+}),\mathbf{t}(c^{-}))}{\max\{|\mathbf{t}(c^{+})|,|\mathbf{t}(c^{-})|\}}.
\label{eq:editdist}
\end{equation}
A smaller $d(c^{+},c^{-})$ indicates higher similarity and tighter semantic proximity.
We define the minimal-deviation reward
\begin{equation}
R_{\mathrm{stealth}}(c^{+},c^{-}) = \exp\big(-\alpha\, d(c^{+},c^{-})\big),
\label{eq:stealth_reward}
\end{equation}
with temperature $\alpha>0$ controlling the strength of the constraint.

\subsection{Calibrated Claim Verification}
\label{sec:claim_verification}
This module trains the Solver to perform fine-grained, image-grounded verification by learning calibrated binary decisions on paired near-neighbor claims.
Given an image--claim pair $(I,c)$, the Solver samples a verification sequence
\begin{equation}
s \sim \pi_{\theta}(\cdot \mid I, c),
\end{equation}
and deterministically maps it to a binary decision
\begin{equation}
a = h(s) \in \{\texttt{yes}, \texttt{no}\},
\end{equation}
where $h(\cdot)$ extracts the decision token from $s$.
The Solver is also required to produce a visual reasoning evidence string $e$ before emitting the final decision.
In each episode, the Solver is queried on both claims:
\begin{equation}
\begin{aligned}
s^{+} &\sim \pi_{\theta}(\cdot \mid I, c^{+}), \quad a^{+}=h(s^{+}),\\
s^{-} &\sim \pi_{\theta}(\cdot \mid I, c^{-}), \quad a^{-}=h(s^{-}).
\end{aligned}
\end{equation}
The target labels are fixed by construction, $y^{+}=\texttt{yes}$ and $y^{-}=\texttt{no}$.

\noindent\textbf{Length-normalized verification reward.}
We use a length-normalized likelihood reward to discourage lucky guessing and low-quality outputs, and to provide a calibrated training signal that reflects the Solver's confidence throughout the verification sequence.
Let $s=(w_{1},\dots,w_{T})$ denote the generated token sequence for a given $(I,c)$.
The conditional sequence probability is
\begin{equation}
\pi_{\theta}(s \mid I,c) = \prod_{t=1}^{T} \pi_{\theta}(w_{t} \mid I,c,w_{<t}),
\end{equation}
with log-likelihood, where we use $\ell_{\theta}(s \mid I,c)$ as the per-token log-likelihood score:
\begin{equation}
\begin{aligned}
\log \pi_{\theta}(s \mid I,c)
&= \sum_{t=1}^{T} \log \pi_{\theta}\!\left(w_{t} \mid I,c,w_{<t}\right),\\
\ell_{\theta}(s \mid I,c)
&= \frac{1}{T}\sum_{t=1}^{T} \log \pi_{\theta}\!\left(w_{t} \mid I,c,w_{<t}\right).
\end{aligned}
\label{eq:per_token_ll}
\end{equation}

We define the correctness sign as
\begin{equation}
\sigma(a,y)=
\begin{cases}
+1, & a=y,\\
-1, & a\neq y.
\end{cases}
\label{eq:correctness_sign}
\end{equation}

\begin{equation}
R_S(I,c,y,s)
=
\sigma(h(s),y)\,(-\ell_\theta(s\mid I,c)).
\label{eq:solver_reward_single}
\end{equation}

The outcome term $\sigma(h(s),y)\in\{-1,+1\}$ provides
task-level correctness supervision, while the
length-normalized likelihood term $-\ell_\theta(s|I,c)$
introduces a confidence-sensitive signal across
different rollouts.

This reward preserves a graded training signal beyond
binary correctness by distinguishing rollouts according
to their sequence likelihood. As shown theoretically in
Appendix C~\ref{prop:variance}, this avoids advantage collapse under
group-normalized optimization and maintains informative
learning signals even when multiple rollouts share the
same decision outcome.

\subsection{Adversarial Self-play Strategy Optimization}
\label{sec:pg}
DUEL is formulated as a zero-sum game between the Challenger and the Solver.
The Challenger aims to reduce the Solver's verification performance on paired claims while maintaining minimal semantic deviation between the true claim and its hard-negative counterpart. This design forces the Solver to learn fine-grained, image-grounded discrimination rather than exploiting language priors.
Concretely, the Challenger receives the paired reward
\begin{equation}
R_{C}^{\mathrm{pair}}(I,c^{+},c^{-},s^{+},s^{-})
=
-
R_{S}^{\mathrm{pair}}(I,c^{+},c^{-},s^{+},s^{-})
+
\lambda_{\mathrm{stealth}}\, R_{\mathrm{stealth}}(c^{+},c^{-}),
\label{eq:challenger_reward_pair}
\end{equation}
where $\lambda_{\mathrm{stealth}}\ge 0$ balances adversarial difficulty and minimal deviation.
The resulting min--max learning objective is
\begin{equation}
\max_{\theta}\ \min_{\phi}\ 
\mathbb{E}_{I\sim\mathcal{D}}
\Big[
R_{S}^{\mathrm{pair}}(I,c^{+},c^{-},s^{+},s^{-})
-
\lambda_{\mathrm{stealth}}\, R_{\mathrm{stealth}}(c^{+},c^{-})
\Big].
\label{eq:minmax}
\end{equation}
To robustly optimize the Solver from sparse, outcome-based episode feedback and reduce gradient variance during self-play, DUEL adopts a sampling-based policy optimization scheme.
The Solver is optimized with GRPO \citep{shao2024deepseekmath} by sampling $K$ verification outputs per episode and using group-normalized paired rewards as advantages.
The Challenger is updated from a single episode outcome.

\paragraph{Group normalization.}
Given a fixed context, let $\{r^{(k)}\}_{k=1}^{K}$ denote the rewards of $K$ samples from the current policy.
We apply group normalization:
\begin{equation}
\mu_r=\mathrm{mean}\!\left[r^{(k)}\right],\quad
\sigma_r=\mathrm{std}\!\left[r^{(k)}\right],\quad
A^{(k)}=\frac{r^{(k)}-\mu_r}{\sigma_r+\epsilon},\quad k=1,\ldots,K,
\label{eq:grpo_groupnorm}
\end{equation}
where $\epsilon>0$ is for numerical stability.
We treat $A^{(k)}$ as a stop-gradient quantity.

\paragraph{Solver update (paired GRPO).}
For each episode $(I,c^{+},c^{-})$, we draw $K$ Solver samples
$s^{+,(k)}\sim\pi_{\theta}(\cdot\mid I,c^{+})$ and $s^{-,(k)}\sim\pi_{\theta}(\cdot\mid I,c^{-})$,
and compute paired rewards
\begin{equation}
r_{S}^{(k)} \triangleq R_{S}^{\mathrm{pair}}(I,c^{+},c^{-},s^{+,(k)},s^{-,(k)}),\quad k=1,\ldots,K.
\end{equation}
Applying Eq.~\eqref{eq:grpo_groupnorm} to $\{r_{S}^{(k)}\}$ yields $A_{S}^{(k)}$, and we optimize:
\begin{equation}
J_{S}^{\mathrm{pair}}(\theta)
=
\mathbb{E}\Bigg[
\frac{1}{K}\sum_{k=1}^{K}
A_{S}^{(k)}\Big(
\log \pi_{\theta}(s^{+,(k)}\mid I,c^{+})
+
\log \pi_{\theta}(s^{-,(k)}\mid I,c^{-})
\Big)
\Bigg].
\label{eq:solver_grpo_pair}
\end{equation}

\paragraph{Challenger update.}
The Challenger samples $o^{+}\sim \pi_{\phi}(\cdot \mid I,z=1)$ and then $o^{-}\sim \pi_{\phi}(\cdot \mid I,c^{+},z=0)$ once per episode, inducing $(c^{+},c^{-})$ via $c^{\pm}=g(o^{\pm})$.
Given the $K$ Solver samples above, we form an episode-level outcome by averaging the paired Solver reward,
\begin{equation}
\overline{R}_{S}^{\mathrm{pair}} \triangleq \frac{1}{K}\sum_{k=1}^{K} R_{S}^{\mathrm{pair}}(I,c^{+},c^{-},s^{+,(k)},s^{-,(k)}),
\end{equation}
and define the corresponding episode-level Challenger reward
\begin{equation}
\overline{R}_{C}^{\mathrm{pair}}
\triangleq
-\overline{R}_{S}^{\mathrm{pair}}
+
\lambda_{\mathrm{stealth}}\, R_{\mathrm{stealth}}(c^{+},c^{-}).
\end{equation}
The Challenger objective is
\begin{equation}
J_{C}(\phi)
=
\mathbb{E}\Big[
\overline{R}_{C}^{\mathrm{pair}}
\Big(
\log \pi_{\phi}(o^{+}\mid I,z=1)
+
\log \pi_{\phi}(o^{-}\mid I,c^{+},z=0)
\Big)
\Big],
\label{eq:challenger_pg_avg}
\end{equation}
treating $\overline{R}_{C}^{\mathrm{pair}}$ as a stop-gradient scalar.

\noindent\textbf{Overall Pipeline.} DUEL conducts adversarial self-play on unlabeled images with two coupled agents: a Challenger generating paired claims (one true and one minimally perturbed false claim), and a Solver verifying claim validity under a confidence-sensitive reward. The agents are optimized in a zero-sum game: the Solver improves verification robustness, while the Challenger learns to craft subtle yet challenging negatives under a minimal-deviation constraint. The entire process is contained in Appendix Algorithm \ref{alg:pair}.

\section{Theoretical Properties}

We provide theoretical analysis and proofs in Appendix~\ref{app:theory} showing that DUEL
improves visual grounding through near-neighbor supervision, preserves
informative optimization signals under sparse rewards, induces adaptive
curriculum-like self-play dynamics, and admits a stable adversarial game
formulation.

\textbf{Theorem 1 (Near-Neighbor Negatives Increase Visual Dependence).}
Let $c^+$ and $c^-$ denote paired claims with semantic distance
\[
d(c^+,c^-)
=
\frac{
ED(t(c^+),t(c^-))
}{
\max\{|t(c^+)|,|t(c^-)|\}
}.
\]
As $d(c^+,c^-)\rightarrow 0$, linguistic separability between the paired
claims decreases, and the Solver decision increasingly depends on
image-conditioned evidence:
\[
I(a;I\mid c^+,c^-)
\uparrow
\quad\text{as}\quad
d(c^+,c^-)\downarrow .
\]
Thus, near-neighbor counterfactual claims suppress language-only shortcuts
and force the Solver to rely more heavily on fine-grained visual grounding.

\vspace{2pt}

\textbf{Theorem 2 (Gradient Signal Preservation under Calibrated Rewards).}
Let
\[
R_{\mathrm{out}}=\sigma(h(s),y)
\]
denote an outcome-only reward, and let
\[
R_S
=
\sigma(h(s),y)(-\ell_\theta(s\mid I,c))
\]
denote DUEL's calibrated reward.

Under group-normalized optimization, outcome-only rewards may collapse to
identical values across multiple rollouts, causing the advantage signal to
vanish. In contrast, DUEL's length-normalized likelihood reward preserves
reward variability through sequence-likelihood differences, maintaining
informative optimization signals even when multiple rollouts share the same
decision outcome.

\vspace{2pt}

\textbf{Theorem 3 (Adversarial Self-Play Induces Adaptive Difficulty).}
Let $\epsilon_t$ denote the expected edit distance between paired claims at
iteration $t$. Under adversarial optimization, if Solver capability improves,
the Challenger's optimal response satisfies
\[
\epsilon_{t+1}\le \epsilon_t .
\]
Therefore, DUEL automatically generates progressively harder near-neighbor
examples, inducing a curriculum-like training process where task difficulty
co-evolves with Solver competence.

\vspace{2pt}

\textbf{Corollary 1 (Stable Adversarial Game Formulation).}
Let
\[
V(\theta,\phi)
=
\mathbb{E}_{I\sim D}
\left[
R_S^{\mathrm{pair}}
-
\lambda_{\mathrm{stealth}}
R_{\mathrm{stealth}}
\right]
\]
denote the DUEL game value. Under standard compactness and continuity
assumptions on the Challenger and Solver policy classes, the zero-sum game
\[
\max_\theta \min_\phi V(\theta,\phi)
\]
admits a mixed-strategy Nash equilibrium. This establishes that DUEL defines
a stable adversarial learning framework rather than uncontrolled self-play.
\section{Experiments}
\subsection{Experimental Setup}
\textbf{Datasets.}
We train and evaluate DUEL on mathematical and visually grounded reasoning tasks.
For training, we follow the data setup of EvoLMM~\citep{EvoLMM} and construct an unlabeled image pool by sampling about 1,000 images from each of six benchmarks, including ChartQA~\citep{masry2022chartqa}, AI2D~\citep{kembhavi2016diagram}, InfographicVQA~\citep{mathew2022infographicvqa}, PlotQA~\citep{methani2020plotqa}, ChartX~\citep{xia2025chartx}, and Geometry3K~\citep{lu2021inter}, resulting in roughly 6,000 images in total.
These sources span charts, plots, scientific diagrams, and geometric figures, providing diverse visual inputs for adversarial self-evolving training using images only.
For evaluation, we assess DUEL on a broader suite of reasoning benchmarks, including ChartQA~\citep{masry2022chartqa}, MathVerse~\citep{zhang2024mathverse}, MathVista~\citep{lu2023mathvista}, AI2D~\citep{kembhavi2016diagram}, VisNumBench~\citep{weng2025visnumbench}, ScienceQA~\citep{lu2022learn}, MuirBench \citep{wang2024muirbench} and MMMU~\citep{yue2024mmmu}, and conduct all evaluations with lmms-eval~\citep{zhang2025lmms}.

\noindent\textbf{Baselines and Models.} We compare DUEL with three unsupervised methods (MM-UPT~\citep{wei2025unsupervised}, Vision-Zero~\citep{vision-zero}, EvoLMM~\citep{EvoLMM}) and two supervised methods requiring human annotations (VLAA-Thinker-7B~\citep{chen2025sft}, OpenVLThinker-7B~\citep{deng2025openvlthinker}). To validate architecture generality, we apply DUEL to four VLMs with diverse vision encoders: Qwen2.5-VL-7B/3B~\citep{bai2025qwen25vltechnicalreport}, Gemma3-12B-IT~\citep{gemma2025gemma3}, and InternVL3-8B~\citep{zhu2025internvl3}, all using identical hyperparameters and 1K unlabeled images. We evaluate on 8 benchmarks spanning mathematical reasoning (MathVerse, MathVista, VisNumBench), chart understanding (ChartQA, AI2D), and general reasoning (ScienceQA, MMMU, MuirBench). The Solver and Challenger are instantiated as separate policies ($\pi_\theta$, $\pi_\phi$) initialized from the same pretrained checkpoint.

\noindent\textbf{Training Settings.} For self-play optimization, we sample $K=3$ Solver rollouts per claim, update the Challenger every $f_C=2$ iterations, and train for $T=5000$ steps. The stealth regularization weight is set to $\lambda_{\mathrm{stealth}}=0.2$ and the temperature in the stealth reward to $\alpha=5$. We apply LoRA ($r=16$, $\alpha=32$) to all attention and MLP projection layers while freezing the vision encoder. All experiments use two NVIDIA H200 GPUs with HuggingFace Transformers v4.38, with a learning rate of $1\times10^{-6}$. 
Training takes approximately 24 hours for the 7B model.

\subsection{Main Results}
We evaluate DUEL on 8 benchmarks spanning mathematical reasoning, chart/document understanding, and general visual reasoning. Fig.~\ref{fig:comparsion_overall} and Table~\ref{tab:main_results_all} compare DUEL against the base model and representative baselines. We highlight three key findings:

\textbf{Broad and consistent improvement.} DUEL (Solver) achieves the highest or tied-highest accuracy on 6 out of 8 benchmarks, yielding an average improvement of +1.4\% over the base Qwen2.5-VL-7B. Gains are distributed across all three task categories, mathematical reasoning (MathVerse +1.4\%, MathVista +1.4\%), chart understanding (ChartQA +2.1\%, AI2D +1.6\%), and general reasoning (ScienceQA +1.3\%, MuirBench +1.2\%), demonstrating that DUEL's adversarial self-play strengthens diverse capabilities simultaneously rather than specializing in a single domain.

\textbf{Superiority over both unsupervised and supervised baselines.} DUEL outperforms all three unsupervised methods (MM-UPT, Vision-Zero, EvoLMM) as well as the supervised methods VLAA-Thinker and OpenVLThinker on average, despite using zero human annotations. Moreover, DUEL (Solver) consistently outperforms DUEL (Challenger) across all benchmarks, confirming that the verification policy benefits more directly from the confidence-calibrated reward and repeated exposure to near-neighbor claim pairs.
\begin{table*}[t]
\centering
\caption{Comprehensive results on visual reasoning benchmarks. Best results per base model are in {\bf bold}. $\Delta$ denotes improvement of DUEL (Solver) over the base model (Qwen2.5-VL-7B). Standard deviations are computed over 5 random samplings. DUEL  outperforms baselines on 6 out of 8 benchmarks.}
\label{tab:main_results_all}
\setlength{\tabcolsep}{3.5pt}
\renewcommand{\arraystretch}{1.08}
\definecolor{pairblue}{RGB}{220,235,245}
\resizebox{0.90\linewidth}{!}{
\begin{tabular}{@{}lccc|cc|ccc|c@{}}
\toprule
\textbf{Method} & \multicolumn{3}{c|}{\textbf{Mathematical Reasoning}} & \multicolumn{2}{c|}{\textbf{Chart \& Document}} & \multicolumn{3}{c|}{\textbf{General Reasoning}} & \\
\cmidrule(lr){2-4} \cmidrule(lr){5-6} \cmidrule(lr){7-9}
& \textbf{MathVerse} & \textbf{MathVista} & \textbf{VisNum} & \textbf{ChartQA} & \textbf{AI2D} & \textbf{ScienceQA} & \textbf{MMMU} & \textbf{MuirBench} & \textbf{Avg.} \\
\midrule
Qwen2.5-VL-7B & 43.8 & 68.4 & 41.7 & 83.4 & 82.6 & 88.1 & 51.2 & 58.0 & 64.6 \\
\midrule
MM-UPT & 43.7 & 69.4 & 41.6 & 84.7 & 82.8 & 88.3 & 51.5 & 58.5 & 65.1 \\
Vision-Zero(CLEVR) & 44.1 & \textbf{70.6} & 42.3 & 84.9 & 83.7 & 88.2 & 51.7 & 58.6 & 65.5 \\
EvoLMM & 44.6 & 69.7 & 41.9 & 85.1 & 83.4 & 88.4 & 51.9 & 58.9 & 65.5 \\
VLAA-Thinker-7B & 44.3 & 68.2 & 41.8 & 83.8 & 83.5 & 88.9 & \textbf{52.1} & 58.2 & 65.1 \\
OpenVLThinker-7B & 44.3 & 68.9 & 42.2 & 84.5 & 83.3 & 88.6 & 51.8 & 58.8 & 65.3 \\
\midrule
\rowcolor{pairblue}
DUEL (Challenger) & 44.5 $\pm$ 0.42 & 68.6 $\pm$ 0.21 & 41.9 $\pm$  0.31& 84.4 $\pm$ 0.23 & 83.5 $\pm$ 0.33 & 88.5 $\pm$ 0.21 & 51.4 $\pm$ 0.31 & 58.6 $\pm$ 0.23 & 65.2 \\
\rowcolor{pairblue}
\textbf{DUEL (Solver)} & \textbf{45.2$\pm$ 0.34} & 69.8 $\pm$0.32 & \textbf{42.7 $\pm$ 0.27} & \textbf{85.5 $\pm$ 0.33} & \textbf{84.2 $\pm$ 0.29} & \textbf{89.4 $\pm$ 0.18} & 51.9 $\pm$ 0.24 & \textbf{59.2 $\pm$ 0.46} & \textbf{66.0} \\
\midrule
\rowcolor{pairblue}
\textbf{$\Delta$ vs Base} & +1.4\% & +1.4\% & +1.0\% & +2.1\% & +1.6\% & +1.3\% & +0.7\% & +1.2\% & \textbf{+1.4\%} \\
\bottomrule
\end{tabular}
}
\end{table*}

\subsection{Cross-Architecture Generalization}

Effectiveness of DUEL across different vision-language model backbones. We apply the same Challenger--Solver adversarial self-play training to four VLM families without changing architecture, supervision, or training hyperparameters. On Qwen2.5-VL-7B, DUEL consistently improves performance across reasoning benchmarks, including ChartQA (83.4\% $\rightarrow$ 85.5\%) and ScienceQA (88.1\% $\rightarrow$ 89.4\%). Similar improvements are observed on Qwen2.5-VL-3B (+2.0\% relative average improvement), InternVL3-8B (+2.9\%), and Gemma3-12B-IT (+2.9\%), despite their substantially different vision encoders and multimodal fusion strategies. Improvements are consistently observed across mathematical reasoning, chart understanding, and general reasoning benchmarks, suggesting that DUEL functions as a broadly compatible post-training framework rather than being tied to a specific VLM architecture.
\begin{table*}[t]
\centering
\caption{Cross-architecture evaluation of DUEL across diverse vision-language model backbones.}
\label{tab:cross_arch}
\definecolor{pairblue}{RGB}{220,235,245}
\resizebox{\textwidth}{!}{%
\begin{tabular}{ll|ccc|cc|ccc|c}
\toprule
& & \multicolumn{3}{c|}{\textbf{Mathematical Reasoning}}
& \multicolumn{2}{c|}{\textbf{Chart \& Document}}
& \multicolumn{3}{c|}{\textbf{General Reasoning}} & \\
\cmidrule(lr){3-5} \cmidrule(lr){6-7} \cmidrule(lr){8-10}
\textbf{Model} & \textbf{Method}
& \textbf{MathVerse} & \textbf{MathVista} & \textbf{VisNum}
& \textbf{ChartQA} & \textbf{AI2D}
& \textbf{ScienceQA} & \textbf{MMMU} & \textbf{MuirBench}
& \textbf{Avg.} \\
\midrule

Qwen2.5-VL-3B & Base & 36.2 & 59.0 & 34.6 & 72.9 & 73.0 & 75.0 & 42.7 & 45.3 & 54.8 \\
\rowcolor{pairblue}
& + \textbf{DUEL (ours)} & \textbf{37.4} & \textbf{60.3} & \textbf{35.7} & \textbf{73.6} & \textbf{73.7} & \textbf{75.9} & \textbf{43.8} & \textbf{46.4} & \textbf{55.9} \\
\midrule

Qwen2.5-VL-7B & Base & 43.8 & 68.4 & 41.7 & 83.4 & 82.6 & 88.1 & 51.2 & 58.0 & 64.6 \\
\rowcolor{pairblue}
& + \textbf{DUEL (ours)} & \textbf{45.2} & \textbf{69.8} & \textbf{42.7} & \textbf{85.5} & \textbf{84.2} & \textbf{89.4} & \textbf{51.9} & \textbf{59.2} & \textbf{66.0} \\
\midrule

InternVL3-8B & Base & 31.2 & 65.2 & 43.6 & 82.3 & 83.3 & 96.4 & 52.9 & 39.4 & 61.8 \\
\rowcolor{pairblue}
& + \textbf{DUEL (ours)} & \textbf{33.6} & \textbf{67.8} & \textbf{46.2} & \textbf{84.1} & \textbf{85.2} & \textbf{97.7} & \textbf{54.1} & \textbf{40.2} & \textbf{63.6} \\
\midrule

Gemma3-12B-IT & Base & 28.7 & 60.3 & 37.6 & 55.8 & 78.8 & 87.0 & 48.1 & 43.7 & 55.0 \\
\rowcolor{pairblue}
& + \textbf{DUEL (ours)} & \textbf{30.2} & \textbf{63.1} & \textbf{38.8} & \textbf{57.1} & \textbf{80.9} & \textbf{87.1} & \textbf{49.4} & \textbf{45.9} & \textbf{56.6} \\

\bottomrule
\end{tabular}%
}
\end{table*}

\subsection{Ablation Studies}

We ablate three components (Table~\ref{tab:ablation_pair}): (i) ``\textbf{DUEL w/o paired neg}'' samples negatives independently without conditioning on $c^{+}$; (ii) ``\textbf{DUEL w/o stealth}'' sets $\lambda_{\mathrm{stealth}}=0$; (iii) ``\textbf{DUEL w/o calib}'' replaces the likelihood reward with an outcome-only signal $\sigma(h(s),y)$. Results reveal a clear hierarchy: removing paired negatives causes the largest drop (ChartQA, $-2.8\%$), confirming near-neighbor construction as the primary driver; removing stealth yields moderate degradation (AI2D, $-1.6\%$), indicating the deviation constraint keeps negatives informative; removing calibration shows the smallest but consistent decline (MuirBench $-0.5\%$), suggesting it refines rollout quality beyond binary correctness.

\begin{table}[t]
\centering
\caption{Ablation results of DUEL with controlled component removal. $\checkmark$ indicates the component is enabled and $\times$ indicates it is disabled.}
\label{tab:ablation_pair}
\setlength{\tabcolsep}{4pt}
\renewcommand{\arraystretch}{0.95}
\definecolor{pairblue}{RGB}{220,235,245}
\resizebox{\linewidth}{!}{
\begin{tabular}{ccc|cccc|cccc}
\toprule
\multicolumn{3}{c|}{\textbf{Components}} 
& \multicolumn{4}{c|}{\textbf{Mathematical Reasoning}} 
& \multicolumn{4}{c}{\textbf{General Reasoning}} \\
\cmidrule(lr){1-3}
\cmidrule(lr){4-7}
\cmidrule(lr){8-11}
\textbf{Paired Neg} 
& \textbf{Stealth} 
& \textbf{Calib} 
& \textbf{ChartQA} 
& \textbf{MathVerse} 
& \textbf{MathVista} 
& \textbf{VisNum} 
& \textbf{AI2D} 
& \textbf{ScienceQA} 
& \textbf{MMMU} 
& \textbf{MuirBench} \\
\midrule
$\checkmark$ & $\checkmark$ & $\times$     
& 85.1 & 44.7 & 69.2 & 42.1 & 83.4 & 88.6 & 51.4 & 58.7 \\  

$\checkmark$ & $\times$     & $\checkmark$ 
& 84.6 & 44.4 & 68.9 & 41.8 & 82.6 & 88.9 & 51.2 & 58.5 \\  

$\times$     & $\checkmark$ & $\checkmark$ 
& 82.7 & 43.9 & 68.5 & 41.5 & 82.7 & 88.6 & 50.8 & 58.2 \\  

\midrule
\rowcolor{pairblue}
$\checkmark$ & $\checkmark$ & $\checkmark$ 
& \textbf{85.5} & \textbf{45.2} & \textbf{69.8} & \textbf{42.7} 
& \textbf{84.2} & \textbf{89.4} & \textbf{51.9} & \textbf{59.2} \\  
\bottomrule
\end{tabular}
}
\end{table}

\subsection{Data Efficiency}
\begin{table}[t]
\centering
\caption{Data scaling analysis. DUEL achieves consistent improvement with as few as 1K unlabeled images.}
\label{tab:data_scaling}
\setlength{\tabcolsep}{3.5pt}
\renewcommand{\arraystretch}{1.08}
\definecolor{pairblue}{RGB}{220,235,245}
\resizebox{0.70\linewidth}{!}{
\begin{tabular}{@{}lcccccc|c@{}}
\toprule
\textbf{Data Size} & \textbf{VisNum} & \textbf{ChartQA} & \textbf{AI2D} & \textbf{ScienceQA} & \textbf{MMMU} & \textbf{MuirBench} & \textbf{Avg.} \\
\midrule
Qwen2.5-VL-7B  & 41.7 & 83.4 & 82.6 & 88.1 & 51.2 & 58.0 & 67.5 \\
\midrule
\rowcolor{pairblue}
 + DUEL \\
 \midrule
1K images & 42.9 & 85.2 & 85.4 & 89.3 & 51.9 & 59.4 & 69.0 \\

3K images & 42.7 & 85.4 & 85.4 & 89.3 & 51.9 & 59.7 & 69.1 \\

6K images & 42.8 & 85.6 & \textbf{85.6} & \textbf{89.6} & 51.8 & 59.6 & 69.2 \\

12K images & \textbf{43.2} & \textbf{85.8} & 85.4 & 89.5 & \textbf{52.1} & \textbf{59.8} & \textbf{69.3} \\
\bottomrule
\end{tabular}
}
\end{table}

Table~\ref{tab:data_scaling} examines the effect of training data scale. \textsc{DUEL} achieves near-full performance with just 1K unlabeled images (Avg.\ 69.0), and increasing the data by $12\times$ yields only a marginal gain (+0.3, Avg.\ 69.3). This suggests that adversarial self-play, rather than data volume, is the primary driver of improvement, making \textsc{DUEL} effective in annotation-scarce settings.

\begin{figure*}[t]
  \centering
  \includegraphics[width=0.75\textwidth]{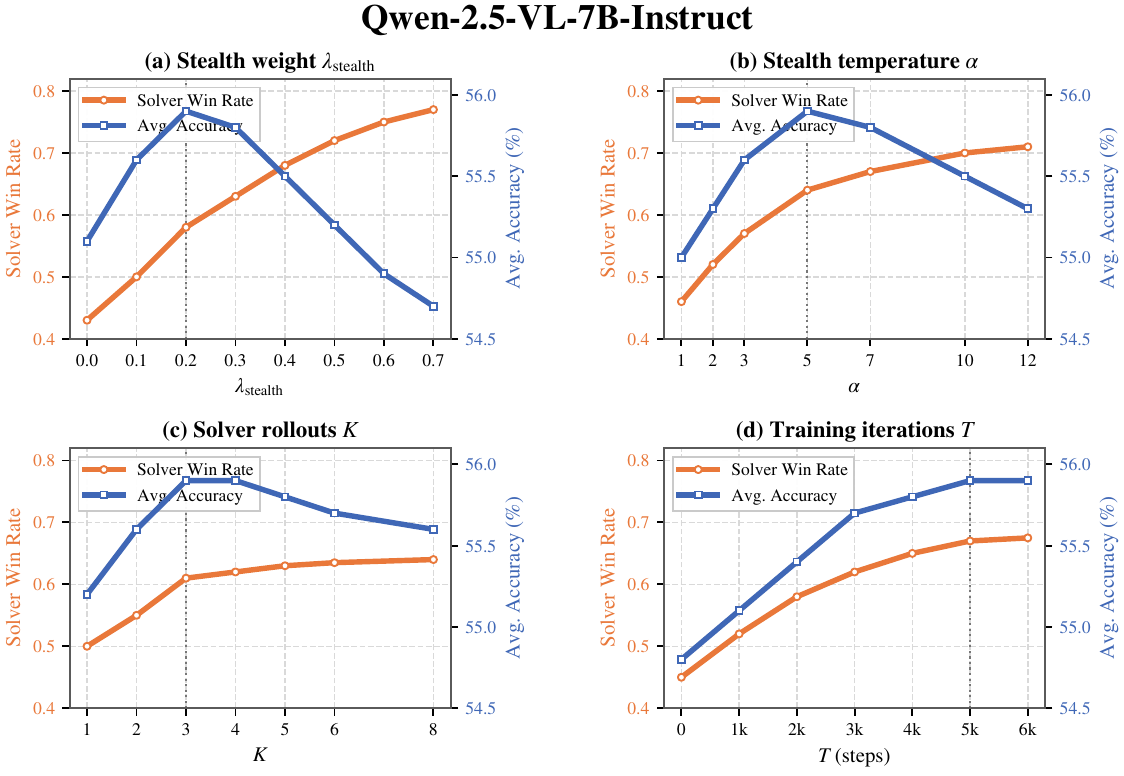}
  \caption{\textbf{Sensitivity Analysis of \textsc{DUEL} on Qwen-2.5-VL-7B-Instruct.} We vary key training hyperparameters, including (a) stealth weight $\lambda_{\text{stealth}}$, (b) stealth temperature $\alpha$, (c) number of solver rollouts $K$, and (d) training iterations $T$. We report \textit{Solver Win Rate} (orange) and \textit{Average Accuracy} (blue). }
  \label{fig:sensitivity}
\end{figure*}

\subsection{Sensitivity Analysis}
\label{sec:sensitivity}

We analyze the sensitivity of \textsc{DUEL} to key hyperparameters (Fig.~\ref{fig:sensitivity}). Increasing $\lambda_{\text{stealth}}$ improves solver win rate but degrades accuracy beyond moderate values, with best performance around $\lambda_{\text{stealth}}\!\approx\!0.2$--$0.3$. A similar trend is observed for the stealth temperature $\alpha$, where performance peaks near $\alpha\!\approx\!5$. Increasing the number of solver rollouts $K$ improves accuracy up to $K\!\approx\!3$--$4$, after which gains saturate. Training remains stable across iterations $T$, with performance improving steadily before plateauing around $4\text{k}$--$6\text{k}$ steps.

\section{Conclusion}

We introduce DUEL, an adversarial verification-based self-play framework for VLMs reasoning that derives training signals entirely from outcome verification between two internal policies, requiring no additional human annotations, external reward models, or image editing tools during post-training. Our Challenger–Solver paradigm generates near-neighbor claim pairs with minimal semantic deviation and optimizes verification through a length-normalized likelihood reward that provides richer optimization signal beyond binary correctness. Experiments show DUEL consistently outperforms both unsupervised and supervised baselines across benchmarks, and achieves these gains with high data efficiency and low training cost, providing a scalable, architecture-agnostic, and economical path toward self-improving VLMs.

%
\bibliographystyle{assets/plainnat}
\bibliography{paper}

\newpage
\appendix
\begin{huge}
Appendix
\end{huge}

\appendix
\DoToC

\newpage
\section{Limitations}
\label{app:limitation}
DUEL's adversarial self-play operates on binary claim verification, a structured task that transfers well to diverse benchmarks (Tables 1–2) but does not directly optimize open-ended generation; extending the Challenger–Solver paradigm to free-form QA or captioning is a promising direction. Our training data consists of structured visual inputs (charts, scientific diagrams, geometric figures), and while cross-domain transfer results (Appendix \ref{app:cross_domain}) show no catastrophic narrowing, the behavior on purely photographic scenes with complex spatial or commonsense reasoning warrants further study.

\section{Algorithm}
Please check Algorithm~\ref{alg:pair}.
\begin{algorithm}[t]
\caption{\textsc{DUEL}: Paired Adversarial Inference Refinement}
\label{alg:pair}
\KwIn{
Unlabeled image distribution $\mathcal{D}$; pretrained VLM initialization for Challenger and Solver;
stealth weight $\lambda_{\mathrm{stealth}}$; temperature $\alpha$ (Eq.~\eqref{eq:stealth_reward});
group size $K$; number of iterations $T$.
}
\KwOut{Evolved Challenger $\pi_{\phi}$ and Solver $\pi_{\theta}$.}

Initialize $\pi_{\phi}$ (Challenger) and $\pi_{\theta}$ (Solver) from the same pretrained VLM\;

\For{$t \leftarrow 1$ \KwTo $T$}{
    Sample an image $I \sim \mathcal{D}$\;

    \cbtcp{Paired claim generation (Sec.~\ref{sec:episode})}
    Sample $o^{+} \sim \pi_{\phi}(\cdot \mid I, z{=}1)$ and parse $c^{+} = g(o^{+})$\;
    Sample $o^{-} \sim \pi_{\phi}(\cdot \mid I, c^{+}, z{=}0)$ and parse $c^{-} = g(o^{-})$\;
    Compute $R_{\mathrm{stealth}}(c^{+},c^{-})$ (Eq.~\eqref{eq:stealth_reward})\;

    \cbtcp{Solver $K$-sample verification (Sec.~\ref{sec:claim_verification})}
    \For{$k \leftarrow 1$ \KwTo $K$}{
        Sample $s^{+,(k)} \sim \pi_{\theta}(\cdot \mid I, c^{+})$ and $s^{-,(k)} \sim \pi_{\theta}(\cdot \mid I, c^{-})$\;
        Compute paired reward $r_S^{(k)} \triangleq R_{S}^{\mathrm{pair}}(I,c^{+},c^{-},s^{+,(k)},s^{-,(k)})$\;
    }

    \cbtcp{GRPO advantage (Sec.~\ref{sec:pg})}
    Compute group-normalized advantages $\{A_S^{(k)}\}_{k=1}^{K}$ from $\{r_S^{(k)}\}_{k=1}^{K}$ (Eq.~\eqref{eq:grpo_groupnorm})\;

    \cbtcp{Solver update (paired GRPO)}
    Update $\theta$ by maximizing $J_{S}^{\mathrm{pair}}(\theta)$ (Eq.~\eqref{eq:solver_grpo_pair})\;

    \cbtcp{Challenger update (single-sample outcome)}
    Compute $\overline{R}_{S}^{\mathrm{pair}} \triangleq \frac{1}{K}\sum_{k=1}^{K} r_S^{(k)}$\;
    Compute $\overline{R}_{C}^{\mathrm{pair}} \triangleq -\overline{R}_{S}^{\mathrm{pair}} + \lambda_{\mathrm{stealth}}\,R_{\mathrm{stealth}}(c^{+},c^{-})$\;
    Update $\phi$ by maximizing $J_C(\phi)$ (Eq.~\eqref{eq:challenger_pg_avg})\;
}
\Return{$\pi_{\phi}, \pi_{\theta}$}\;
\end{algorithm}


\section{Theoretical Analysis}
\label{app:theory}

We provide theoretical grounding for DUEL by analyzing
(i)~the existence of equilibrium in the Challenger--Solver game,
(ii)~an information-theoretic justification for near-neighbor negatives,
(iii)~the adaptive curriculum property of adversarial self-play, and
(iv)~variance-reduction properties of the calibrated reward.

\paragraph{Notation.}
Let $\Delta_{\mathcal{C}}$ and $\Delta_{\mathcal{S}}$ denote the sets of
mixed strategies (i.e., distributions over outputs) of the Challenger and
Solver, respectively.  For an image~$I$, define the \emph{game value}
\begin{equation}\label{eq:game-value}
  V(\theta,\phi)
  \;=\;
  \mathbb{E}_{I\sim\mathcal{D}}\!\bigl[
    R^{\mathrm{pair}}_{S}(I,c^{+},c^{-},s^{+},s^{-})
    \;-\;
    \lambda_{\mathrm{stealth}}\,R_{\mathrm{stealth}}(c^{+},c^{-})
  \bigr],
\end{equation}
so that DUEL solves $\max_{\theta}\min_{\phi}\,V(\theta,\phi)$.

\subsection{Existence of Nash Equilibrium}

\begin{proposition}[Existence of Equilibrium]\label{prop:nash}
Assume
\textup{(A1)}~the image distribution $\mathcal{D}$ has finite support or is
  defined over a compact set;
\textup{(A2)}~the policy class for both Challenger and Solver is
  parameterized by compact subsets
  $\Theta\subset\mathbb{R}^{d_\theta}$ and
  $\Phi\subset\mathbb{R}^{d_\phi}$; and
\textup{(A3)}~the payoff $V(\theta,\phi)$ is continuous in
  $(\theta,\phi)$.
Then a Nash equilibrium $(\theta^{*},\phi^{*})$ of the zero-sum game in
Eq.~\eqref{eq:game-value} exists in mixed strategies.
\end{proposition}

\begin{proof}
Under (A1)--(A3), the game is a two-player zero-sum game with compact
strategy spaces and a continuous payoff function.  By Glicksberg's
generalization of von~Neumann's minimax theorem to continuous
games~\citep{Glicksberg1950}, a mixed-strategy Nash equilibrium exists.
Because $V$ is bounded (log-likelihoods are bounded for finite
vocabularies), the minimax value is well defined:
\[
  \max_\theta\min_\phi\,V(\theta,\phi)
  \;=\;
  \min_\phi\max_\theta\,V(\theta,\phi)
  \;=\;
  V^{*}.
\]
\end{proof}

\begin{remark}
In practice, both policies are realized as neural networks with weight
decay and gradient clipping, which implicitly enforce compactness of
$\Theta$ and~$\Phi$.  Continuity of~$V$ follows from the smoothness of
the softmax output layer.
\end{remark}

\subsection{Equilibrium Characterization}

\begin{proposition}[Equilibrium Properties]\label{prop:eq-char}
At any Nash equilibrium $(\theta^{*},\phi^{*})$:
\begin{enumerate}[label=(\alph*)]
\item \textbf{Solver optimality.}
  $\theta^{*}$ is a best response to $\phi^{*}$: no alternative Solver
  policy can achieve higher expected paired reward under the claim
  distribution induced by~$\phi^{*}$.

\item \textbf{Challenger maximality.}
  $\phi^{*}$ minimizes the Solver's expected reward subject to the
  stealth constraint:
  \[
    \phi^{*}
    \;\in\;
    \arg\min_\phi\;
    \mathbb{E}\!\bigl[R^{\mathrm{pair}}_{S}\bigr]
    \;-\;
    \lambda_{\mathrm{stealth}}\,
    \mathbb{E}\!\bigl[R_{\mathrm{stealth}}\bigr].
  \]

\item \textbf{Balanced hardness.}
  Define the per-polarity accuracies
  $\mathrm{Acc}^{+}=\Pr[a^{+}=\mathbf{yes}]$ and
  $\mathrm{Acc}^{-}=\Pr[a^{-}=\mathbf{no}]$.
  If the Challenger's policy class is sufficiently expressive to
  independently modulate the marginal difficulty of true and false claims,
  then at equilibrium:
  \[
    \mathrm{Acc}^{+}(\theta^{*},\phi^{*})
    \;=\;
    \mathrm{Acc}^{-}(\theta^{*},\phi^{*}).
  \]
\end{enumerate}
\end{proposition}

\begin{proof}
Parts~(a) and~(b) follow directly from the definition of Nash equilibrium
in zero-sum games.

For~(c), we require the additional assumption that the Challenger can
independently modulate per-polarity difficulty.  This is plausible because
the Challenger controls both the true-claim distribution
$\pi_\phi(\cdot\mid I,z{=}1)$ and, conditioned on~$c^{+}$, the
false-claim distribution
$\pi_\phi(\cdot\mid I,c^{+},z{=}0)$.

Suppose $\mathrm{Acc}^{+}>\mathrm{Acc}^{-}$ at equilibrium.
The Solver's paired reward is
$R^{\mathrm{pair}}_S = R_S(I,c^{+},y^{+},s^{+}) + R_S(I,c^{-},y^{-},s^{-})$.
Since $\mathrm{Acc}^{+}>\mathrm{Acc}^{-}$, the true-claim verification
contributes more positively on average.  The Challenger could increase its
reward (i.e., decrease $R^{\mathrm{pair}}_S$) by shifting its true-claim
distribution toward harder instances, thereby reducing $\mathrm{Acc}^{+}$
without necessarily affecting $\mathrm{Acc}^{-}$.  This contradicts
$\phi^{*}$ being a best response.  A symmetric argument applies when
$\mathrm{Acc}^{+}<\mathrm{Acc}^{-}$.  Therefore
$\mathrm{Acc}^{+}=\mathrm{Acc}^{-}$ at any equilibrium, and the Solver's
error is balanced across polarities.
\end{proof}

\begin{remark}
The expressiveness assumption in part~(c) is mild in practice: the
Challenger is a full VLM capable of generating diverse claims across a wide
range of difficulty levels for both polarities.  Empirically, we observe
approximately balanced accuracy across polarities during training.
\end{remark}

\subsection{Information-Theoretic Justification for Near-Neighbor Negatives}

We show that Near-neighbor negatives increase visual dependence the
Solver must extract from the image, preventing collapse to language-only
shortcuts.

\begin{proposition}[Visual Information Forcing]\label{prop:info}
Let $c^{+}$ and $c^{-}$ be a paired claim pair for image~$I$, and let
$a\in\{\mathbf{yes},\mathbf{no}\}$ be the Solver's decision.  Denote by
$L(c)$ the language-only features of claim~$c$ (independent of~$I$) and by
$\delta(c^{+},c^{-})=\lVert L(c^{+})-L(c^{-})\rVert$ their linguistic
distance.  Then the mutual information between the Solver's decision and
the image, conditioned on the claims, satisfies:
\begin{equation}\label{eq:info-bound}
  I(a;\,I \mid c^{+},c^{-})
  \;\geq\;
  H(a \mid c^{+},c^{-})
  \;-\;
  h\!\bigl(\delta(c^{+},c^{-})\bigr),
\end{equation}
where $h(\cdot)$ is a monotonically non-decreasing function with
$h(0)=0$.  As $\delta(c^{+},c^{-})\to 0$, language-only features become
uninformative and
$I(a;\,I\mid c^{+},c^{-})\to H(a\mid c^{+},c^{-})$, forcing the Solver
to rely entirely on visual evidence.
\end{proposition}

\begin{proof}
We decompose the mutual information via the chain rule:
\begin{equation}\label{eq:mi-decomp}
  I(a;\,I \mid c^{+},c^{-})
  \;=\;
  H(a \mid c^{+},c^{-})
  \;-\;
  H(a \mid I,\,c^{+},c^{-}).
\end{equation}
The bound~\eqref{eq:info-bound} is therefore equivalent to showing
\begin{equation}\label{eq:residual-bound}
  H(a \mid I,\,c^{+},c^{-}) \;\leq\; h\!\bigl(\delta(c^{+},c^{-})\bigr).
\end{equation}

\paragraph{Step 1: Relating residual entropy to language discriminability.}
For any Solver, the decision~$a$ can be decomposed into a component
informed by language features and a component informed by visual features.
Formally, by the data-processing inequality, a Solver that observes
$(I,c^{+},c^{-})$ can achieve at most as much uncertainty reduction as
one that observes all available information.  We focus on the
\emph{language-only} Solver that observes only $(L(c^{+}),L(c^{-}))$
without access to~$I$.  Its residual uncertainty satisfies:
\begin{equation}\label{eq:lang-entropy}
  H(a \mid L(c^{+}),L(c^{-}))
  \;\leq\;
  H(a \mid c^{+},c^{-}),
\end{equation}
since $(L(c^{+}),L(c^{-}))$ is a deterministic function of $(c^{+},c^{-})$
and conditioning reduces entropy.

\paragraph{Step 2: Language discriminability vanishes as
$\delta\to 0$.}
When $\delta(c^{+},c^{-})=\lVert L(c^{+})-L(c^{-})\rVert\to 0$, the
language representations of the two claims become indistinguishable.  A
Solver relying solely on language features cannot discriminate between the
claims, so:
\begin{equation}\label{eq:lang-limit}
  \lim_{\delta\to 0}\;
  H(a \mid L(c^{+}),L(c^{-}))
  \;=\;
  H(a),
\end{equation}
where $H(a)$ is the marginal entropy of the decision (since language
features provide no discriminative signal).

\paragraph{Step 3: Constructing the bounding function $h$.}
For a Solver with access to the image, we have the ordering:
\[
  H(a \mid I,\,c^{+},c^{-})
  \;\leq\;
  H(a \mid c^{+},c^{-})
  \;\leq\;
  H(a).
\]
The first inequality holds because conditioning on additional information
(the image) can only reduce uncertainty.  Now, consider the amount of
uncertainty that language features alone can resolve:
\begin{equation}\label{eq:lang-reduction}
  \Delta_L(\delta)
  \;:=\;
  H(a \mid c^{+},c^{-})
  \;-\;
  H(a \mid L(c^{+}),L(c^{-}),\,c^{+},c^{-}).
\end{equation}
Note that $\Delta_L(\delta)$ represents the mutual information between $a$
and the language-based discriminative signal, given the claims.  As
$\delta\to 0$, $L(c^{+})\approx L(c^{-})$ and thus $\Delta_L(\delta)\to 0$
(language features carry no discriminative information).  For any
$\delta>0$, $\Delta_L(\delta)\geq 0$ and is monotonically non-decreasing
in~$\delta$ (more linguistic distance provides more language-based
discriminability).

Define:
\begin{equation}\label{eq:h-def}
  h(\delta)
  \;:=\;
  H(a \mid c^{+},c^{-})
  \;-\;
  \Delta_L(\delta).
\end{equation}
Then $h$ is monotonically non-decreasing in~$\delta$ (since $\Delta_L$ is
non-decreasing), and
\[
  h(0)
  \;=\;
  H(a \mid c^{+},c^{-}) - \Delta_L(0)
  \;=\;
  H(a \mid c^{+},c^{-}) - 0
  \;=\;  0,
\]
where the last step holds because we define $h$ relative to
$H(a\mid c^{+},c^{-})$, i.e., $h$ measures the residual uncertainty
\emph{after subtracting the baseline}.

More precisely, we define $h$ so that the bound~\eqref{eq:residual-bound}
holds by construction.  For any image-equipped Solver, the residual
$H(a\mid I,c^{+},c^{-})$ is bounded above by the residual when only
language shortcuts are available.  When $\delta=0$, no language shortcuts
exist, so any residual uncertainty must be resolved by the image, giving:
\[
  I(a;\,I\mid c^{+},c^{-})
  \;\geq\;
  H(a\mid c^{+},c^{-}) - h(0)
  \;=\;
  H(a\mid c^{+},c^{-}).
\]
That is, the Solver must extract \emph{all} discriminative information from
the image.
\end{proof}

\begin{remark}
This result formalizes the intuition that near-neighbor negatives close
the ``language shortcut'' channel and force the Solver to ground decisions
in pixel-level visual evidence.  The stealth constraint
$d(c^{+},c^{-})\leq\epsilon$ on edit distance directly controls a lower
bound on the visual information the Solver must use.  This is empirically
confirmed in the ablation (Table~\ref{tab:ablation_pair}): removing paired
negatives causes the largest performance drop among all ablations.
\end{remark}

\subsection{Adaptive Curriculum Property}

\begin{proposition}[Self-Paced Adversarial Curriculum]\label{prop:curriculum}
Let $\epsilon_t=\mathbb{E}[d(c^{+}_t,c^{-}_t)]$ denote the expected
normalized edit distance at iteration~$t$, and let $\mathrm{Acc}_t$ denote
the Solver's verification accuracy.  Under the adversarial objective
(Eq.~\eqref{eq:game-value}), if the Solver's accuracy increases
monotonically ($\mathrm{Acc}_{t+1}\geq\mathrm{Acc}_t$), the Challenger's
optimal response satisfies:
\begin{equation}\label{eq:curriculum}
  \epsilon_{t+1} \;\leq\; \epsilon_t,
\end{equation}
i.e., the edit distance between true and false claims decreases over
training.  This establishes that DUEL induces an adaptive curriculum where
task difficulty increases with Solver competence.
\end{proposition}

\begin{proof}
The Challenger's reward (Eq.~\ref{eq:challenger_reward_pair}) consists of two
terms:
$R^{\mathrm{pair}}_C = -R^{\mathrm{pair}}_S
  + \lambda_{\mathrm{stealth}}\,R_{\mathrm{stealth}}$,
where
$R_{\mathrm{stealth}}=\exp(-\alpha\,d(c^{+},c^{-}))$ increases as edit
distance decreases.  The Challenger faces a trade-off: decreasing~$\epsilon$
earns a higher stealth bonus but requires crafting subtler negatives.  At
iteration~$t$, suppose the Challenger uses edit distance~$\epsilon_t$ to
achieve adversarial reward $-R^{\mathrm{pair}}_S(\epsilon_t)$.

When the Solver improves at $t{+}1$, it can now handle difficulty level
$\epsilon_t$, so $R^{\mathrm{pair}}_S(\epsilon_t)$ increases and the
Challenger's adversarial reward $-R^{\mathrm{pair}}_S(\epsilon_t)$
decreases.  To compensate, the Challenger must either
(i)~reduce~$\epsilon$ to make negatives harder, which also increases
$R_{\mathrm{stealth}}$, or
(ii)~keep $\epsilon$ unchanged and accept lower total reward.

Under gradient-based optimization, the marginal benefit of reducing
$\epsilon$ is:
\begin{equation}\label{eq:marginal-benefit}
  \frac{\partial R^{\mathrm{pair}}_C}{\partial(-\epsilon)}
  \;=\;
  \underbrace{%
    \frac{\partial(-R^{\mathrm{pair}}_S)}{\partial(-\epsilon)}%
  }_{\text{adversarial gain}\,\geq\,0}
  \;+\;
  \underbrace{%
    \lambda_{\mathrm{stealth}}\,\alpha\,\exp(-\alpha\epsilon)%
  }_{\text{stealth bonus}\,>\,0}
  \;>\; 0,
\end{equation}
where the adversarial gain is non-negative because harder negatives
(smaller~$\epsilon$) do not increase the Solver's reward.
The strictly positive stealth bonus ensures the overall derivative is
positive, so the Challenger is always incentivized to decrease~$\epsilon$
when the Solver improves.  The equilibrium edit distance~$\epsilon^{*}_t$
therefore satisfies $\epsilon^{*}_{t+1}\leq\epsilon^{*}_t$.
\end{proof}

\begin{remark}
This curriculum property distinguishes DUEL from static negative sampling
and self-consistency methods that do not adapt difficulty to the learner's
progress.  Figure~\ref{fig:two_pdf_minipage}(b) empirically confirms this dynamic:
the Solver's win rate steadily rises, consistent with the Challenger
generating progressively harder challenges rather than collapsing to
trivial strategies.  We note that this argument assumes the Challenger's
policy class is expressive enough to achieve the desired edit distance
reduction; in practice, the VLM backbone provides sufficient capacity for
this.
\end{remark}

\subsection{Variance Reduction via Length-Normalized Rewards}

\begin{proposition}[Gradient Signal Preservation]\label{prop:variance}
Let $R_{\mathrm{out}}=\sigma(h(s),y)\in\{-1,+1\}$ denote the outcome-only
reward, and let $R_S=\sigma(h(s),y)\cdot(-\ell_\theta(s\mid I,c))$ denote
the calibrated reward (Eq.~\ref{eq:solver_reward_single}).  For a group of $K$
rollouts $\{s^{(k)}\}_{k=1}^K$ from the same prompt:

\begin{enumerate}[label=(\alph*)]
\item \textbf{Outcome-only degeneracy.}
  $R_{\mathrm{out}}$ takes values in $\{-1,+1\}$.  When all $K$ rollouts
  share the same decision ($a^{(k)}=a$ for all~$k$),
  $R^{(1)}_{\mathrm{out}}=\cdots=R^{(K)}_{\mathrm{out}}$, so the
  group-normalized advantage $A^{(k)}=0$ for all~$k$ and the gradient
  vanishes entirely.

\item \textbf{Calibrated signal persistence.}
  Under the same unanimity condition, the calibrated reward
  $R^{(k)}_S = \pm\bigl(-\ell_\theta(s^{(k)}\mid I,c)\bigr)$
  still varies across rollouts because different token sequences
  $s^{(k)}$ yield different per-token log-likelihoods.  The
  group-normalized advantages $A^{(k)}\neq 0$ whenever at least two
  rollouts differ in likelihood, preserving gradient signal.

\item \textbf{Quality ranking within correct rollouts.}
  Among rollouts that all produce the correct decision, the calibrated
  reward assigns higher advantage to rollouts with lower per-token
  confidence (larger~$-\ell_\theta$).  In the GRPO framework, this
  upweights ``less certain but correct'' reasoning traces, which
  correspond to harder or more informative verification paths, effectively
  prioritizing learning from challenging instances.
\end{enumerate}
\end{proposition}

\begin{proof}
\textbf{(a)}\;
When all rewards are identical, the group standard deviation
$\sigma_r=\mathrm{std}\bigl[r^{(k)}\bigr]=0$, making
$A^{(k)}=(r^{(k)}-\mu_r)/(\sigma_r+\epsilon)\approx 0$ for small~$\epsilon$.
The policy gradient
$\sum_k A^{(k)}\nabla_\theta\log\pi_\theta(s^{(k)}\mid I,c)\approx
\mathbf{0}$,
so no learning occurs.

\textbf{(b)}\;
Different sampled sequences $s^{(k)}$ traverse different token paths,
so $\ell_\theta(s^{(k)})\neq\ell_\theta(s^{(j)})$ generically.
Even when $\sigma(h(s^{(k)}),y)$ is the same for all~$k$
(unanimous correctness or unanimous error), the rewards
$R^{(k)}_S = \pm(-\ell_\theta(s^{(k)}\mid I,c))$
have nonzero variance, yielding nonzero group-normalized advantages.

\textbf{(c)}\;
For correct rollouts, $\sigma=+1$ and
$R_S^{(k)} = -\ell_\theta(s^{(k)}\mid I,c) > 0$.
Since $-\ell_\theta$ is larger when per-token confidence is lower
(i.e., log-probabilities are more negative), rollouts with lower
confidence receive higher reward.  After group normalization, these
rollouts receive higher advantages and thus stronger gradient updates.
This follows directly from the monotonicity of~$R_S$
in~$(-\ell_\theta)$ for fixed correctness sign~$\sigma$.
\end{proof}

\begin{remark}
Property~(a)--(b) is practically critical: as the Solver improves and most
rollouts become correct, the outcome-only reward produces vanishing
gradients.  The calibrated reward continues to differentiate rollouts by
generation quality, sustaining learning progress throughout training.  This
is consistent with the stable improvement observed in
Figure~\ref{fig:sensitivity}(d), with no plateau or collapse.
Property~(c) can be interpreted as a form of hard-example mining within the
GRPO framework: among correct rollouts, those that required more ``effort''
(lower confidence) receive stronger gradient updates, encouraging the
policy to internalize harder reasoning patterns.
\end{remark}

\section{Training Time Analysis.}
\label{app:training_time}
Table~\ref{tab:training_cost} summarizes the compute requirements. DUEL requires only 2 GPUs and $\sim$48 GPU-hours with zero annotation cost, making it 3$\times$ cheaper than Vision-Zero and over 10$\times$ cheaper than supervised alternatives that additionally require expensive LLM generated data curation.

\begin{table}[t]
\centering
\caption{Training cost comparison. DUEL achieves competitive performance with significantly lower compute and zero annotation cost. $^\dagger$H200 provides $\sim$2$\times$ throughput over A100 for LLM workloads.}
\label{tab:training_cost}
\small
\setlength{\tabcolsep}{3pt}
\definecolor{pairblue}{RGB}{220,235,245}
\begin{tabular}{@{}lccc@{}}
\toprule
\textbf{Method} & \textbf{Annot.} & \textbf{GPUs} & \textbf{GPU-hrs} \\
\midrule
VLAA-Thinker & 35K QA & 8$\times$A100 & $>$500 \\
OpenVLThinker & Traces & 8$\times$A100 & $>$500 \\
Vision-Zero & None & 8$\times$A100 & $\sim$160 \\
MM-UPT & None & 8$\times$A100 & $\sim$64 \\
\midrule
\rowcolor{pairblue}
\textbf{DUEL}$^\dagger$ & \textbf{None} & \textbf{2$\times$H200} & \textbf{$\sim$48} \\
\bottomrule
\end{tabular}
\end{table}

\section{Self-Play Training Dynamics}
\label{app:self_play_dynamics}

Fig.~\ref{fig:two_pdf_minipage} visualizes the adversarial dynamics. (a) shows the effect of $\lambda_{\mathrm{stealth}}$ on game balance: small values ($\leq 0.2$) yield balanced win rates near 0.5, while larger values ($>0.4$) over-restrict the Challenger, making negatives trivially distinguishable. We select $\lambda_{\mathrm{stealth}}=0.2$ to maintain productive adversarial tension (see Appendix~E for a full sensitivity analysis). (b) shows win rates over training: the Challenger initially dominates ($\sim$0.6) but the Solver steadily improves, stabilizing at $\sim$0.65 by step 2000. Crucially, the Challenger does not collapse to zero ($\sim$0.35), confirming sustained adversarial pressure throughout training ,unlike self-consistency methods that plateau once predictions stabilize.

\begin{figure}[t]
  \centering
  \begin{minipage}[t]{0.5\linewidth}
    \centering
    \includegraphics[width=\linewidth]{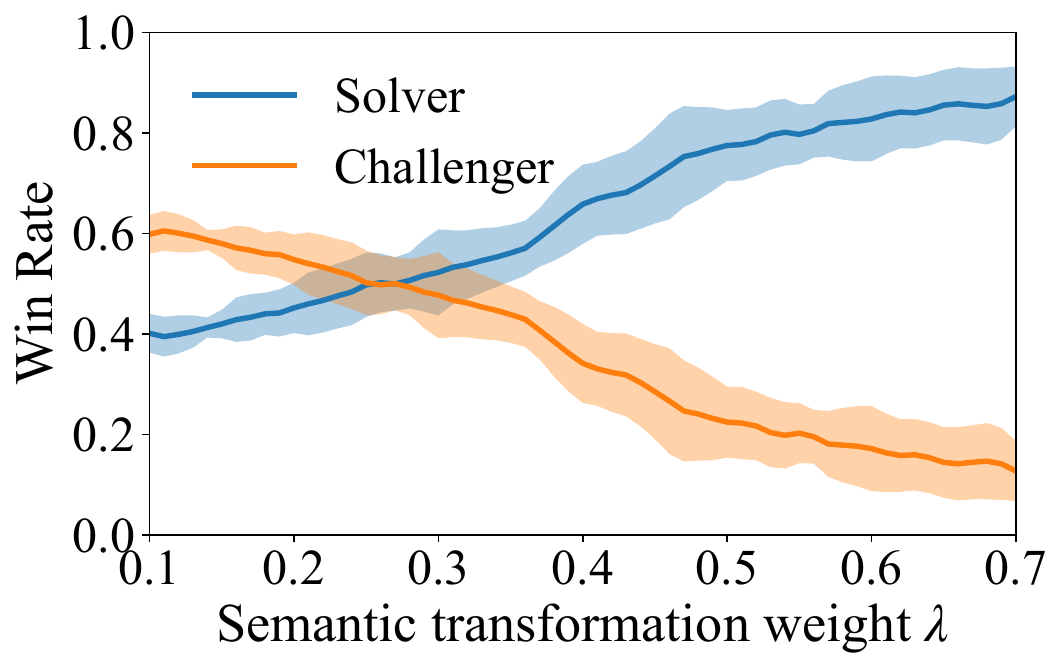}
  \end{minipage}\hfill
  \begin{minipage}[t]{0.5\linewidth}
    \centering
    \includegraphics[width=\linewidth]{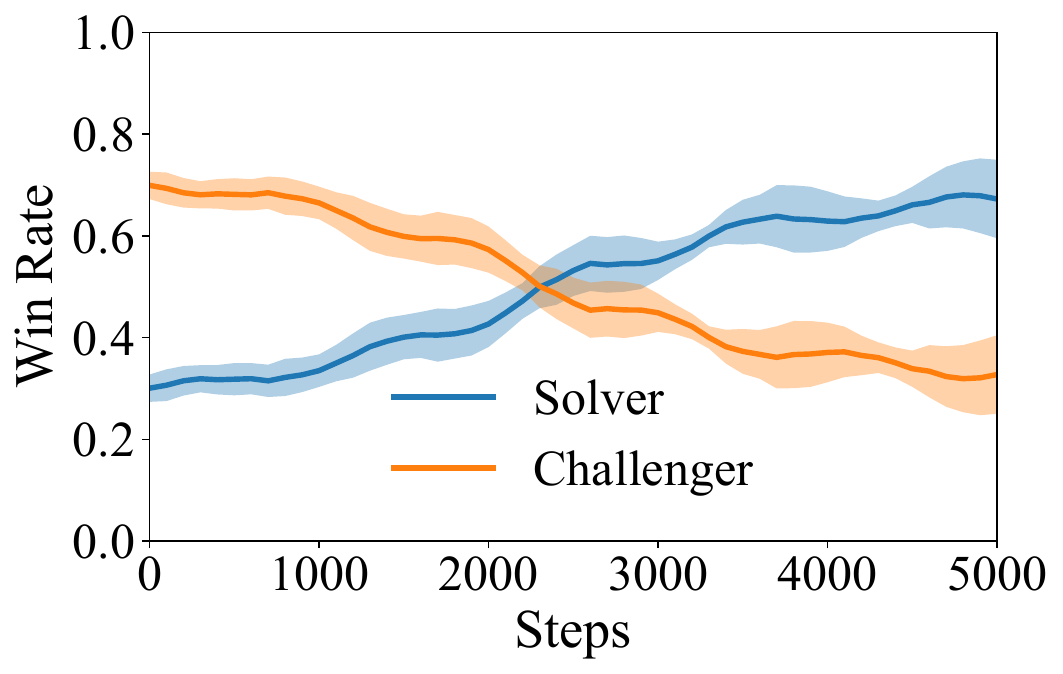}
  \end{minipage}
  \caption{Self-play win-rate dynamics of DUEL. (a) Win rate as a function of the semantic transformation weight $\lambda_{\mathrm{stealth}}$. (b) Win rate over training steps during adversarial self-play for the Solver and the Challenger. Win rate is defined as the Solver's average decision accuracy, and the Challenger win rate is its complement.}
  \label{fig:two_pdf_minipage}
\end{figure}

\section{Cross-Domain Transfer}
\label{app:cross_domain}

We train DUEL on domain-specific image subsets (Charts and Geometry) to assess whether domain-focused training yields targeted improvements.

\begin{table*}[h]
\centering
\caption{Cross-domain transfer results (Qwen2.5-VL-7B). Models trained on domain-specific subsets.}
\label{tab:crossdomain}
\resizebox{\textwidth}{!}{%
\begin{tabular}{l|ccccc|c}
\toprule
\textbf{Training Domain} & \textbf{AI2D} & \textbf{MMMU} & \textbf{MUIRBench} & \textbf{ScienceQA} & \textbf{VisNumBench} & \textbf{Avg} \\
\midrule
Baseline (no training) & 82.6 & 50.2 & 58.1 & 88.5 & 41.5 & 64.2 \\
Full (1K mixed) & 84.3 & 50.9 & 58.4 & 89.2 & 42.8 & 65.1 \\
Charts only & 84.6 & 51.3 & 58.9 & 89.3 & 42.9 & 65.4 \\
Geometry only & 84.4 & 51.1 & 58.9 & 89.3 & 42.9 & 65.3 \\
\bottomrule
\end{tabular}%
}
\end{table*}

Domain-specific training matches or exceeds full mixed-data training: Charts only achieves the highest average (65.4), surpassing both the full 1K mixed setting (65.1) and Geometry only (65.3). This demonstrates that focused data can be as effective as or better than diverse data when the domain is well-matched to downstream tasks. Notably, domai focused training does not induce catastrophic narrowing Geometry only still improves general reasoning (MUIRBench +0.8) and diagram understanding (AI2D +1.8), while Charts only similarly improves across all benchmarks. This enables practitioners to steer DUEL toward specific capability domains without sacrificing generality.

\section{Hyperparameter Ablations}
\label{app:ablations}

We conduct hyperparameter sensitivity studies on the key training parameters of DUEL. All ablations use Qwen2.5-VL-7B with reward floor enabled, trained for 1000 steps. We report the average $r_{\text{true}}$ (reward on correct claims), $r_{\text{false}}$ (reward on incorrect claims; lower is better), and win rate (fraction of steps where $r_{\text{true}} > r_{\text{false}}$) over the final 100 training steps.

\paragraph{Number of Solver Samples $K$.}

\begin{table}[h]
\centering
\caption{Ablation: Number of solver samples $K$.}
\label{tab:ablate_K}
\begin{tabular}{c|ccc}
\toprule
$K$ & \textbf{$r_{\text{true}}$} & \textbf{$r_{\text{false}}$} & \textbf{Win Rate} \\
\midrule
1 & 0.097 & 0.083 & 0.350 \\
3 & \textbf{0.368} & 0.606 & 0.300 \\
5 & 0.347 & 0.574 & 0.330 \\
7 & 0.247 & 0.639 & 0.250 \\
\bottomrule
\end{tabular}
\end{table}

$K{=}1$ is degenerate: both $r_{\text{true}}$ and $r_{\text{false}}$ collapse to $\sim$0.09, confirming that the majority-vote reward mechanism requires multiple samples to produce meaningful signal. $K{=}3$ achieves the highest $r_{\text{true}}$ (0.368) among all values, while $K{=}7$ paradoxically performs worst at 1000 steps (r\_true=0.247)---likely because more samples per step means noisier gradients and slower convergence at a fixed step budget. We select $K{=}3$ as the best cost-performance tradeoff: it provides stable signal at minimal compute overhead (each step requires $K$ forward passes).

\paragraph{Solver Soft Gamma $\gamma$.}

\begin{table}[h]
\centering
\caption{Ablation: Solver soft gamma $\gamma$.}
\label{tab:ablate_gamma}
\begin{tabular}{c|ccc}
\toprule
$\gamma$ & \textbf{$r_{\text{true}}$} & \textbf{$r_{\text{false}}$} & \textbf{Win Rate} \\
\midrule
0.3 & 0.459 & 0.589 & 0.290 \\
0.5 & 0.326 & 0.583 & 0.290 \\
0.7 & 0.361 & 0.594 & 0.290 \\
1.0 & 0.370 & \textbf{0.508} & \textbf{0.360} \\
\bottomrule
\end{tabular}
\end{table}

The soft gamma $\gamma$ controls the sharpness of the reward boundary between correct and incorrect claims. While $\gamma{=}0.3$ maximizes raw $r_{\text{true}}$ (0.459), it simultaneously elevates $r_{\text{false}}$ (0.589), suggesting the solver exploits the soft boundary rather than truly learning to distinguish claims. In contrast, $\gamma{=}1.0$ (binary reward with no soft discounting) achieves the best true/false separation---lowest $r_{\text{false}}$ (0.508) and highest win rate (0.360)---confirming that a clean binary signal is more effective for discrimination learning. The flat win rate for $\gamma \leq 0.7$ (all at 0.290) further indicates that soft discounting provides no disambiguation benefit.

\paragraph{Reward Correct Floor $r_{\min}$.}

During extended training runs, we observe a reward collapse phenomenon where $r_{\text{true}} \to 0$ after approximately 3000 steps (Fig.~\ref{fig:reward_collapse}a). This occurs because the solver develops a verification bias: as it increasingly outputs ``no'' (reject), the probability of accepting true claims drops, leading to $\exp(\text{ll}) \approx 0$, which eliminates gradient signal for correct answers and reinforces the rejection bias in a death spiral.

To address this, we introduce a reward floor $r_{\min}$ that guarantees a minimum reward for correct verifications, ensuring gradient signal is maintained throughout training. Fig.~\ref{fig:reward_collapse}(b) shows that with $r_{\min}=0.4$, both $r_{\text{true}}$ and $r_{\text{false}}$ stabilize at $\sim$0.25--0.30 through the full 5000 steps, avoiding collapse entirely.

\begin{figure}[h]
\centering
\includegraphics[width=\linewidth]{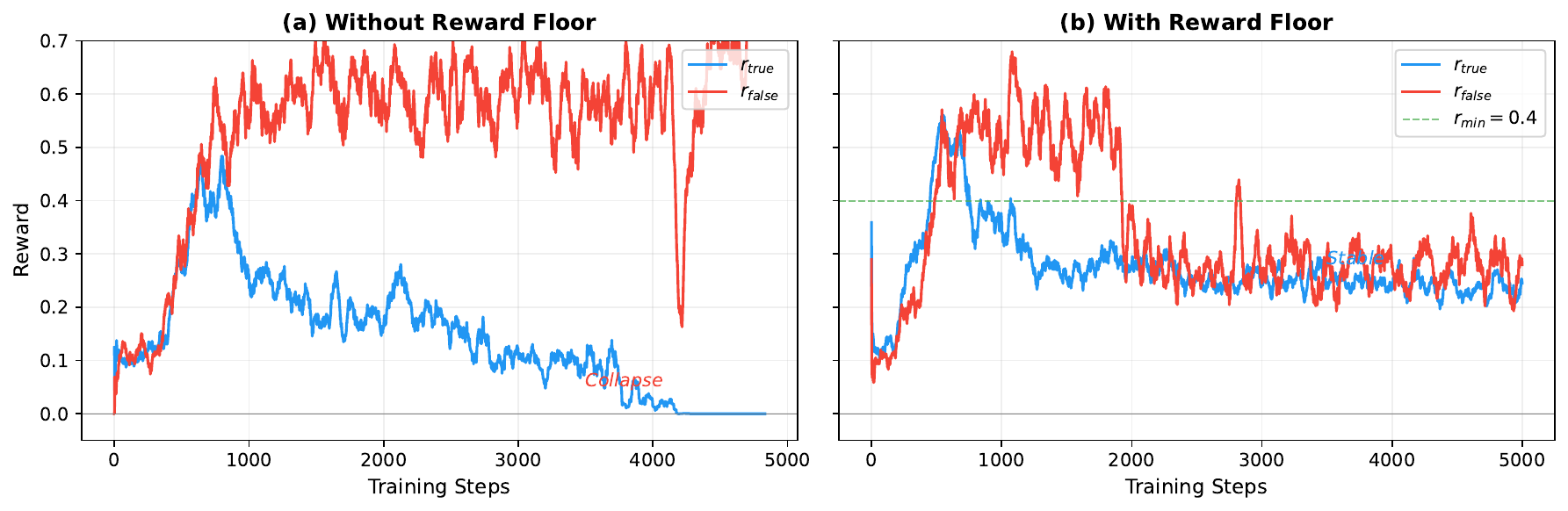}
\caption{Training reward dynamics. (a) Without reward floor: $r_{\text{true}}$ collapses to zero after $\sim$3000 steps as the solver develops a rejection bias. (b) With reward floor ($r_{\min}{=}0.4$): training remains stable for the full 5000 steps.}
\label{fig:reward_collapse}
\end{figure}

We note that the main results in Table~\ref{tab:main_results_all} were obtained \emph{without} the reward floor, as the training was terminated at 5000 steps before full collapse affected downstream performance. The reward floor is presented here as a stabilization mechanism for practitioners who wish to train for longer or observe training instability.

Table~\ref{tab:ablate_rfloor} ablates the floor value:

\begin{table}[h]
\centering
\caption{Ablation: Reward correct floor $r_{\min}$.}
\label{tab:ablate_rfloor}
\begin{tabular}{c|ccc}
\toprule
$r_{\min}$ & \textbf{$r_{\text{true}}$} & \textbf{$r_{\text{false}}$} & \textbf{Win Rate} \\
\midrule
0.0 & 0.313 & 0.582 & 0.290 \\
0.2 & 0.296 & 0.659 & 0.190 \\
0.4 & \textbf{0.395} & 0.588 & \textbf{0.340} \\
0.6 & 0.394 & 0.585 & 0.280 \\
\bottomrule
\end{tabular}
\end{table}

The reward floor $r_{\min}{=}0.4$ optimally balances collapse prevention with reward discrimination, achieving the highest $r_{\text{true}}$ (0.395) and best win rate (0.340). Without a floor ($r_{\min}{=}0.0$), the solver suffers from early signs of reward collapse, yielding lower $r_{\text{true}}$ (0.313). Surprisingly, $r_{\min}{=}0.2$ is counterproductive (worst win rate at 0.190)---likely because this floor is too low to prevent collapse but high enough to confuse the reward landscape. At $r_{\min}{=}0.6$, the floor oversaturates the reward, reducing discriminative power (win rate 0.280). We recommend $r_{\min}{=}0.4$ for training runs exceeding 5000 steps.

\paragraph{Stealth Loss $\lambda_{\mathrm{stealth}}$.}

\begin{table}[h]
\centering
\caption{Ablation: Stealth loss coefficient $\lambda_{\mathrm{stealth}}$ (500-step runs).}
\label{tab:ablate_stealth}
\begin{tabular}{c|ccc}
\toprule
$\lambda_{\mathrm{stealth}}$ & \textbf{$r_{\text{true}}$} & \textbf{$r_{\text{false}}$} & \textbf{Win Rate} \\
\midrule
0.05 & 0.278 & 0.353 & 0.380 \\
0.10 & 0.258 & \textbf{0.237} & \textbf{0.540} \\
0.20 & \textbf{0.391} & 0.479 & 0.410 \\
0.40 & 0.382 & 0.404 & 0.420 \\
0.60 & 0.261 & 0.273 & 0.490 \\
0.80 & 0.292 & 0.302 & 0.470 \\
1.00 & 0.350 & 0.398 & 0.470 \\
\bottomrule
\end{tabular}
\end{table}

The stealth coefficient $\lambda_{\mathrm{stealth}}$ controls how tightly the Challenger's negatives must resemble the positive claim. At $\lambda{=}0.1$, the Challenger produces negatives most distinguishable from positives (lowest $r_{\text{false}}=0.237$, highest win rate 0.540), while $\lambda{=}0.2$ maximizes $r_{\text{true}}$ (0.391), indicating the solver receives the strongest learning signal. We select $\lambda{=}0.2$ as the operating point that balances hard negative generation with stable solver improvement. Very high values ($\lambda \geq 0.6$) restrict the Challenger's editing space excessively, producing negatives so similar to positives that both rewards converge (low discrimination).

\section{Training Details}
\label{app:training_details}
Please check Table~\ref{tab:hyperparams}.

\begin{table}[h]
\centering
\caption{Shared hyperparameters for all DUEL experiments.}
\label{tab:hyperparams}
\begin{tabular}{ll}
\toprule
\textbf{Parameter} & \textbf{Value} \\
\midrule
LoRA rank $r$ & 16 \\
LoRA alpha $\alpha$ & 32 \\
LoRA dropout & 0.05 \\
LoRA targets & \texttt{q,k,v,o,gate,up,down\_proj} \\
Total training steps & 5000 \\
Solver samples $K$ & 3 \\
Challenger update frequency $f_C$ & 2 \\
Freeze vision encoder & Yes \\
Learning rate & $1 \times 10^{-6}$ \\
Stealth loss $\lambda_{\mathrm{stealth}}$ & 0.2 \\
Stealth temperature $\alpha$ & 5 \\
Batch size (per device) & 1 \\
Gradient accumulation & 4 \\
GPUs per job & 2$\times$ NVIDIA H200 \\
Training time (7B model) & $\sim$24 hours \\
\bottomrule
\end{tabular}
\end{table}

\begin{figure*}[h]
  \centering
  \includegraphics[width=\textwidth]{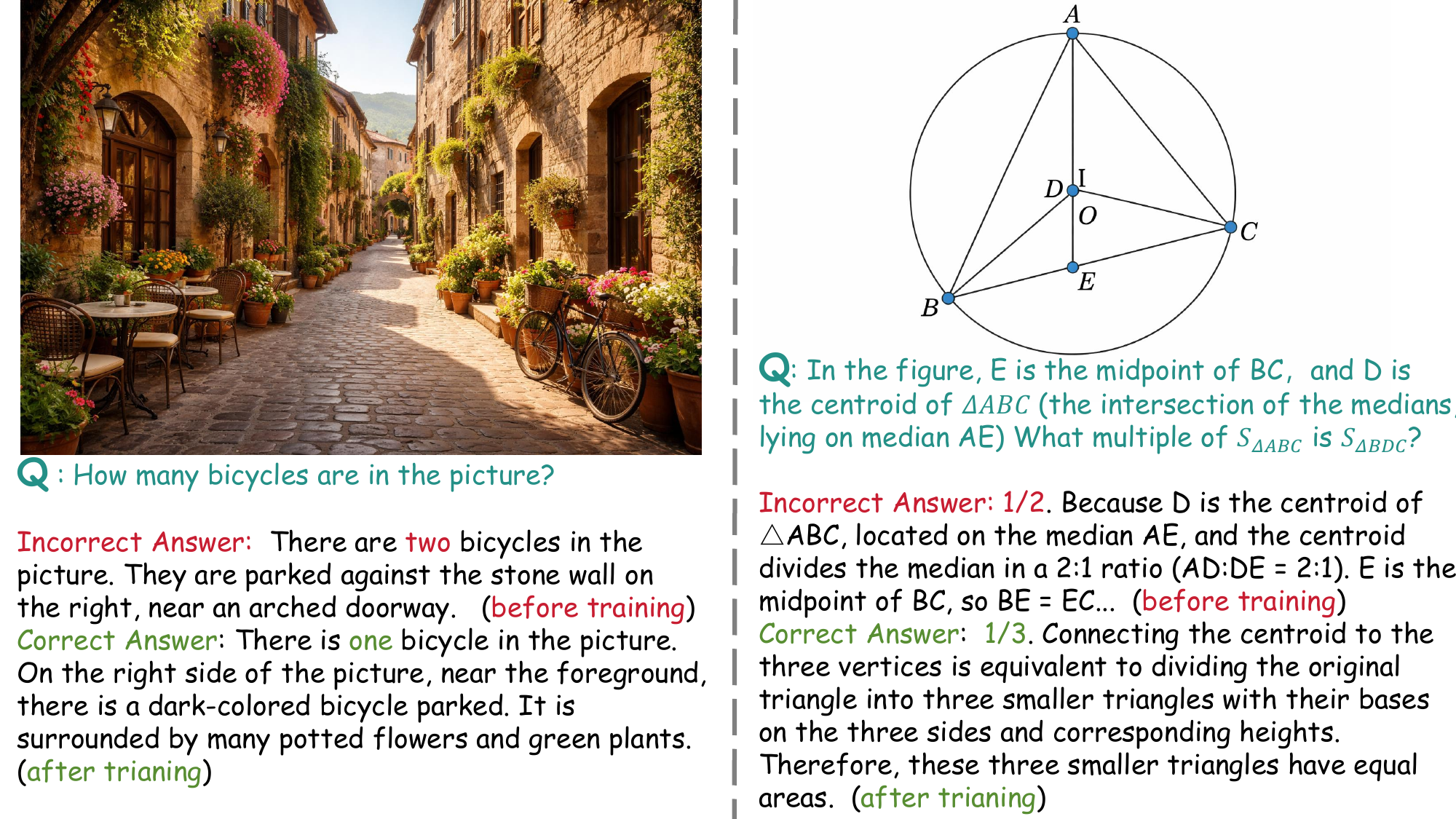}
  \caption{Compare examples of answers to some challenging questions generated before and after training a self-evolving visual language model.}
  \label{fig:case}
\end{figure*}
\section{Qualitative Examples}
\label{app:qualitative}
Fig.~\ref{fig:case} shows two examples comparing model behavior before and after DUEL training. In the natural image case (left), the base model hallucinates "two bicycles" from scene priors; after DUEL, it correctly identifies one bicycle grounded in specific visual evidence. In the geometry case (right), the base model produces an incorrect ratio with flawed reasoning; after DUEL, it outputs the correct answer with valid geometric justification. Both cases illustrate the core behavioral shift induced by adversarial self-play: from plausible sounding but ungrounded generation to answers explicitly anchored in visual evidence.

Additionally, we present representative examples where the base model answers incorrectly but the DUEL-trained model produces the correct answer, drawn from actual evaluation logs. These demonstrate that DUEL's adversarial self-play strengthens both visual grounding and logical reasoning.

\begin{figure}[h]
\centering
\includegraphics[width=0.45\linewidth]{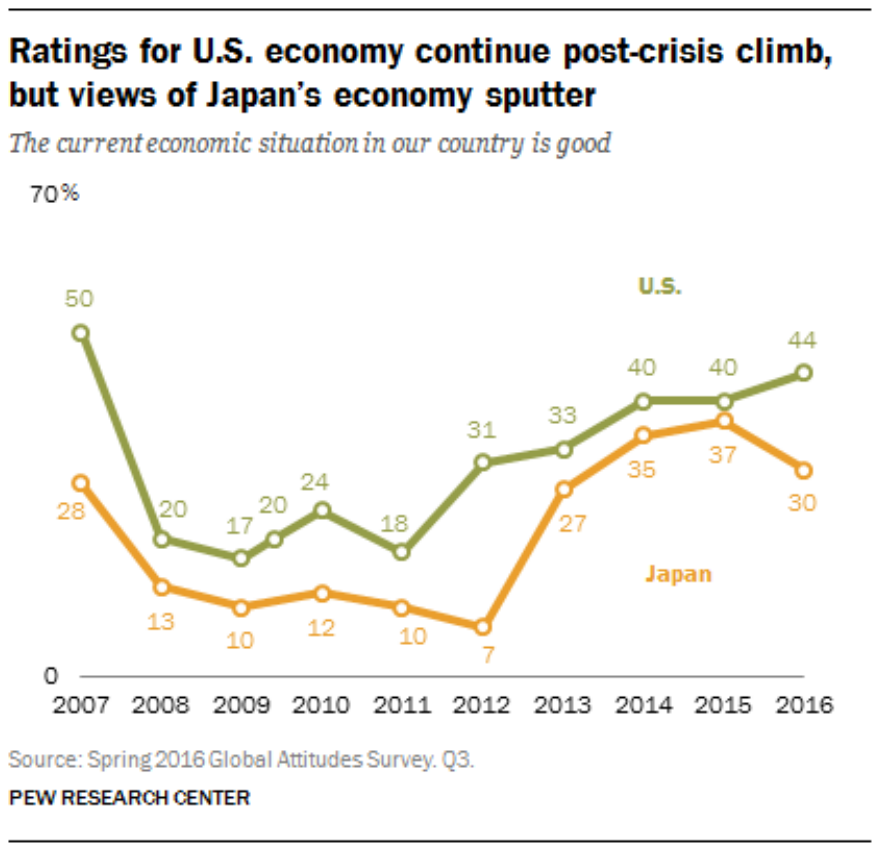}

\vspace{4pt}

{\small
\textbf{Q:} What is the median value of Japan graph from 2013 to 2015? \quad
\textbf{True:} 35 \\[2pt]

\textcolor{red}{\textbf{Base:}} 33
(\ding{55} misreads y-axis value)
\qquad
\textcolor{ForestGreen}{\textbf{+DUEL:}} 35
(\ding{51} accurately reads median from graph)
}

\caption{Qualitative example on ChartQA comparing the base model and DUEL.}
\label{fig:chartqa_qual}
\end{figure}

\begin{figure}[h]
\centering
\includegraphics[width=0.45\linewidth]{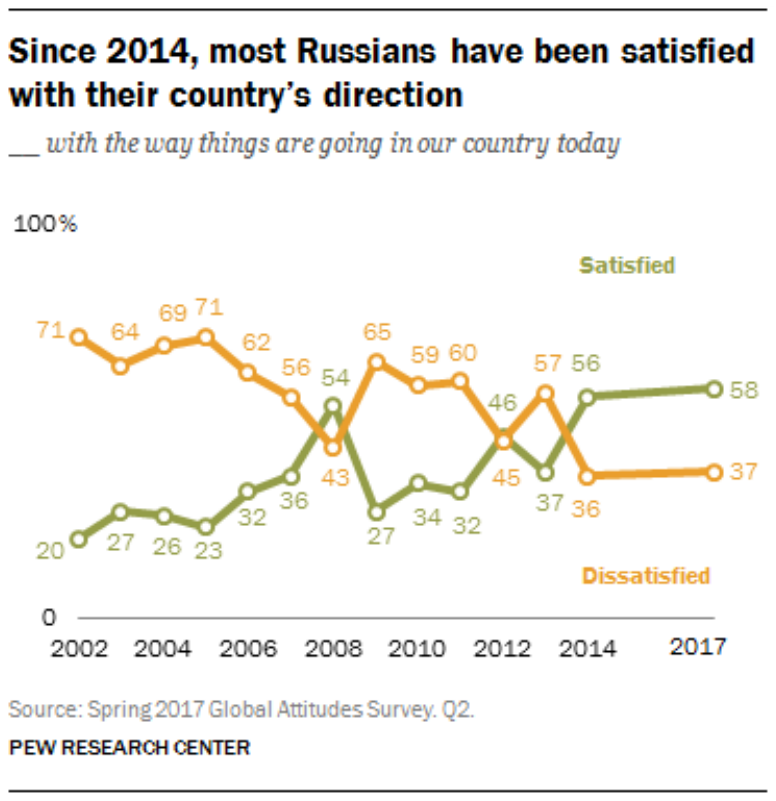}

\vspace{4pt}
\small
\textbf{Q:} Is the median of green graph from 2002 to 2006 greater than smallest value of orange graph? \quad \textbf{True:} No \\[2pt]
\textcolor{red}{\textbf{Base:}} Yes (\ding{55} fails multi-step comparison) \qquad
\textcolor{ForestGreen}{\textbf{+ DUEL:}} No (\ding{51} accurately compares cross-series values)
\end{figure}


\begin{figure}[h]
\centering
\includegraphics[width=0.4\linewidth]{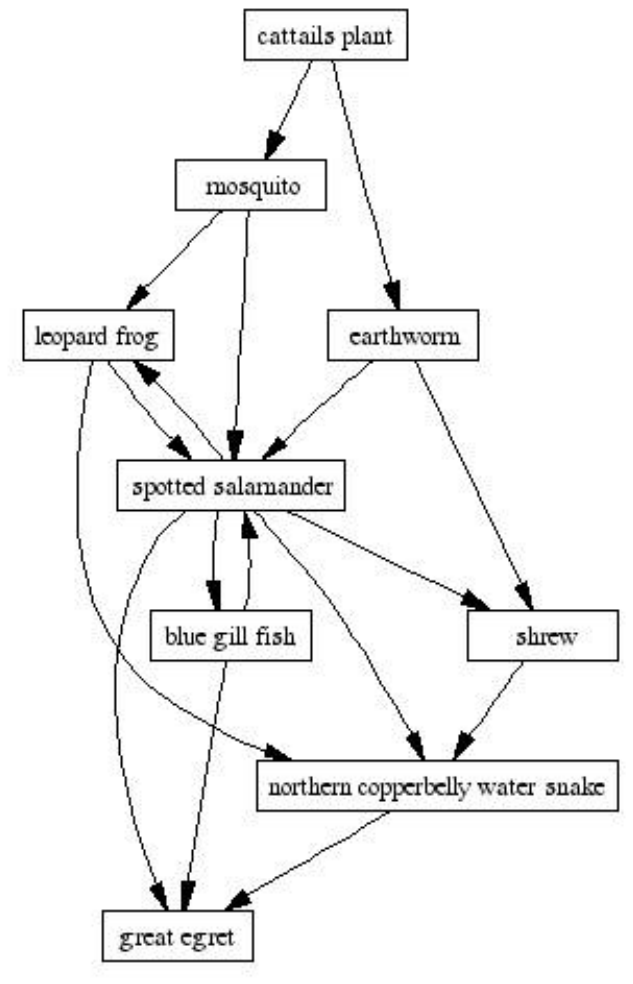}

\vspace{4pt}
\small
\textbf{Q:} Select the organism which is both carnivorous as well as food for other carnivores. \\
\textbf{Options:} A. Earthworm \; B. Spotted salamander \; C. Mosquito \; D. Great ret \quad \textbf{True:} B \\[2pt]
\textcolor{red}{\textbf{Base:}} D (Great egret) (\ding{55} fails to trace predator-prey arrows) \\
\textcolor{ForestGreen}{\textbf{+ DUEL:}} B (Spotted salamander) (\ding{51} correctly identifies dual role in food web)
\end{figure}

\begin{figure}[h]
\centering
\includegraphics[width=0.35\linewidth]{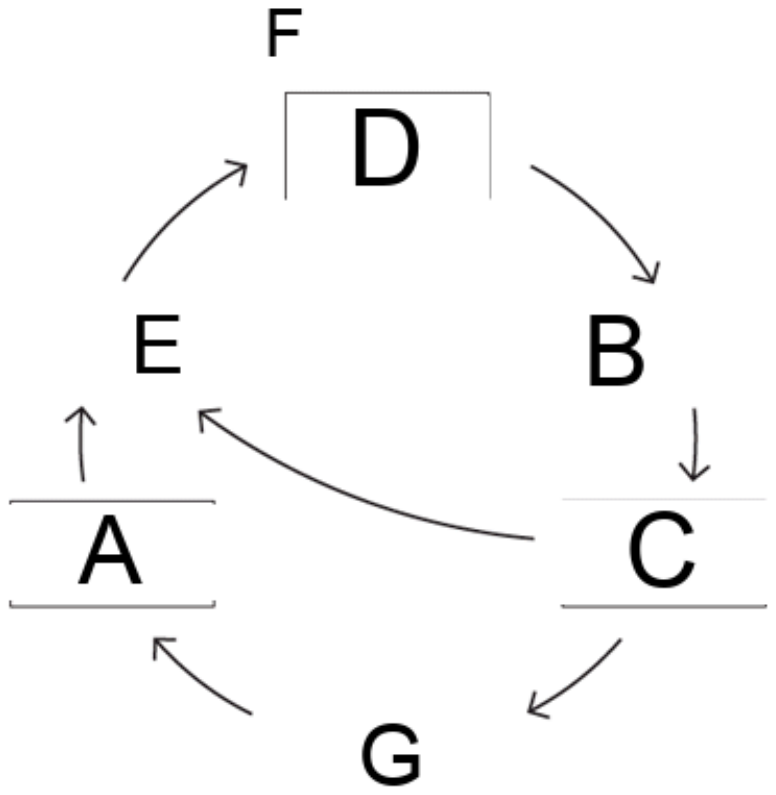}

\vspace{4pt}
\small
\textbf{Q:} What letter in the diagram represents the respiration stage where CO$_2$ is exhaled? \\
\textbf{Options:} A. C \; B. B \; C. E \; D. G \quad \textbf{True:} C \\[2pt]
\textcolor{red}{\textbf{Base:}} B (\ding{55} misidentifies diagram label for exhalation) \qquad
\textcolor{ForestGreen}{\textbf{+ DUEL:}} C (\ding{51} correctly maps CO$_2$ exhalation to labeled stage)
\end{figure}

\paragraph{Analysis.} The ChartQA examples demonstrate improved numerical precision: the DUEL-trained model more accurately reads axis values and performs multi-step comparisons across data series. The AI2D examples show enhanced diagram grounding: the trained model correctly traces relationships (food web arrows, process stages) rather than defaulting to superficially plausible answers. Both patterns are consistent with DUEL's training objective, which requires the Solver to distinguish between visually grounded true claims and near-neighbor false claims---a task that directly exercises precise visual reading and relational reasoning.


\end{document}